\documentclass{article}[preprint]

\usepackage[preprint]{neurips_2024}

\usepackage{amsmath}
\usepackage{amssymb}
\usepackage{mathtools}
\usepackage{amsthm}
\usepackage{xspace}

\usepackage[utf8]{inputenc} 
\usepackage[T1]{fontenc}    
\usepackage{hyperref}       
\usepackage{url}            
\usepackage{booktabs}       
\usepackage{amsfonts}       
\usepackage{nicefrac}       
\usepackage{microtype}      
\usepackage{xcolor}         

\usepackage{amsmath, amsthm, amssymb}
\usepackage{bm}
\usepackage{graphicx}
\usepackage{xcolor}
\usepackage{url}
\usepackage{multirow}

\usepackage{algorithm}
\usepackage{algorithmic}
\usepackage{lipsum}

\title{On the cohesion and separability of average-link for
hierarchical agglomerative clustering}

\author{%
  Eduardo S. Laber \\
   Departmento de Informática, PUC-RIO\\
   \texttt{laber@inf.puc-rio.br} 
   \And Miguel Batista \\
  Departmento de Informática, PUC-RIO\\   \texttt{miguel260503@gmail.com} \\
}

\theoremstyle{plain}
\newtheorem{theorem}{Theorem}[section]
\newtheorem{proposition}[theorem]{Proposition}
\newtheorem{lemma}[theorem]{Lemma}

\theoremstyle{definition}

\theoremstyle{remark}

\usepackage[textsize=tiny]{todonotes}

\newcommand{\goodnessav}{{\tt cs\mbox{-}ratio_{AV}}}
\newcommand{\goodnessdm}{{\tt cs\mbox{-}ratio_{DM}}}

\newcommand{\diamset}{{\tt diam}}
\newcommand{\maxavgset}{{\tt max\mbox{-}avg}}

\newcommand{\maxdiamset}{{\tt max\mbox{-}diam}}
\newcommand{\avgdiamset}{{\tt avg\mbox{-}diam}}
\newcommand{\complink}{{\tt complete-linkage}}
\newcommand{\singlelink}{{\tt single-linkage}}
\newcommand{\avglink}{{\tt average-link}}
\newcommand{\similarity}{{\tt sim}}
\newcommand{\dasgupta}{{\tt Dasg}}
\newcommand{\dasguptadiss}{{\tt CKMM}}

\newtheorem{claim}{Claim}

\newcommand{\remove}[1]{}

\newcommand{\red}[1]{{\color{red} #1}}

\newcommand{\OPT}{\textrm{OPT}\xspace}
\newcommand{\OPTAVG}{\textrm{OPT$_{{\tt AV}}$}\xspace}
\newcommand{\OPTDIM}{\textrm{OPT$_{{\tt DM}}$}\xspace}
\newcommand{\OPTDIMK}{\textrm{OPT$_{{\tt DM}}(k)$}\xspace}
\newcommand{\OPTAVGK}{\textrm{OPT$_{{\tt AV}}(k)$}\xspace}
\newcommand{\OPTSEPK}{\textrm{OPT$_{{\tt SEP}}(k)$}\xspace}

\newcommand{\A}{\mathcal{A}}
\newcommand{\mom}[1]{{\left\vert\kern-0.25ex\left\vert\kern-0.25ex\left\vert #1 \right\vert\kern-0.25ex\right\vert\kern-0.25ex\right\vert}}

\renewcommand{\L}{\mathcal{L}}
\renewcommand{\S}{\mathcal{S}}
\newcommand{\C}{\mathcal{C}}
\newcommand{\T}{\mathcal{T}}

\newcommand{\avg}{{\tt avg}}
\newcommand{\dist}{{\tt dist}}
\newcommand{\diss}{{\tt diss}}
\newcommand{\costavg}{{\tt sep_{av}}}
\newcommand{\costmin}{{\tt sep_{min}}}

\begin{document}

\maketitle

\begin{abstract}
Average-link is widely recognized as one of the most popular
and effective methods
for building hierarchical agglomerative clustering.
The available theoretical analyses  show
that this method has a much better approximation than other popular heuristics,
as single-linkage and complete-linkage, regarding 
variants of Dasgupta's cost function [STOC 2016]. However, these analyses do not
separate  average-link from a random hierarchy and they 
are not appealing for metric spaces since every  hierarchical clustering has
a $1/2$  approximation with regard  to the variant of Dasgupta's function
that is employed for dissimilarity measures [Moseley and Yang 2020].
In this paper, we 
  present a comprehensive study of the performance
of \avglink \, in metric spaces, regarding  several natural criteria that
capture separability and cohesion, and are more interpretable
than Dasgupta's cost function and its variants.
We also present experimental results with real datasets that, together
with our theoretical analyses, suggest that average-link is a better
choice than other related methods when both cohesion and separability are important goals.
\end{abstract}

\section{Introduction}

\remove{Clustering is the problem of partitioning a set of items so that similar items are grouped together and dissimilar items are separated. It is a fundamental tool in machine learning that is commonly used for exploratory analysis and for reducing the computational resources required to handle large datasets. For comprehensive descriptions of different clustering methods and their applications, we refer to \citep{Jain:1999,HMMR15}. }

Clustering is the task of partitioning a set of objects/points so that similar ones are grouped together while dissimilar ones are put in different groups. 
Clustering methods are widely used for exploratory analysis and for reducing the computational resources required to handle large datasets. 


Hierarchical clustering is an important class of clustering methods.
Given a set  of ${\cal X}$ of $n$ points,  a hierarchical clustering is a sequence of
clusterings $({\cal C}^n,{\cal C}^{n-1},\ldots,\C^{1})$,
where ${\cal C}^n$ is a clustering with $n$ unitary clusters, 
each of them corresponding to a point in ${\cal X}$,
and  the clustering ${\cal C}^{i}$, with $i <n $, is obtained 
from ${\cal C}^{i+1}$ by replacing two of its clusters with 
their union $A^i$.
A hierarchical clustering induces a strictly binary tree with $n$ leaves,
where each leaf corresponds to a point in ${\cal X}$
and the $i$th internal node, with $i<n$, is associated with the cluster
$A^i$; the points in $A^i$ correspond to the leaves of the subtree rooted in $A^i$.
Hierarchical clustering methods are often taught in data science/ML courses, are implemented in many machine learning libraries, such as {\tt scipy}, and  have applications in 
different fields as evolution studies via phylogenetic trees \citep{b684a7bdefca4810bdaaf10bf8196c46},
finance \citep{TUMMINELLO201040} and  detection  of closely related entities \citep{DBLP:conf/kdd/KobrenMKM17,DBLP:conf/kdd/MonathDGZAMMNTT21}.


\remove{
There is a significant literature on hierarchical clustering; for  good surveys we refer to
\citep{10.1093/comjnl/26.4.354,DBLP:journals/widm/MurtaghC12}.
With regards to more theoretical work, one important line of research consists of designing algorithms for hierarchical clustering
 with provable guarantees for natural optimization criteria such as cluster diameter and the sum of quadratic errors \citep{DASGUPTA2005555,DBLP:journals/siamcomp/CharikarCFM04,DBLP:journals/siamcomp/LinNRW10,DBLP:conf/esa/ArutyunovaR22}.
Another relevant line aims
to understand the theoretical properties (e.g. approximation guarantees)
of algorithms widely used in practice, such as linkage methods
\citep{DASGUPTA2005555,DBLP:journals/corr/abs-1012-3697,DBLP:conf/esa/GrosswendtR15,arutyunova_et_al:LIPIcs.APPROX/RANDOM.2021.18,DBLP:conf/soda/GrosswendtRS19}.
Here, we contribute to this second line of research by
 presenting  a number of new analysis for the  \avglink \, algorithm.
}

\remove{Despite of its importance, little is known from a theoretical perspective.
It is not even clear what objective function it optimizes.
}

{\tt Average-link} is widely considered one of the most effective
 hierarchical clustering algorithms. 
 It belongs to the class of {\em agglomerative methods},
 that is, methods  that start with a set of $n$ clusters,
corresponding to the $n$ input points, and iteratively use 
a linkage rule to merge two clusters.
Due to its relevance,  we can find some recent works
dedicated to improving \avglink' efficiency and scalability  
\citep{DBLP:journals/pvldb/YuWGDS21,DBLP:conf/icml/DhulipalaELMS21,DBLP:conf/nips/DhulipalaELMS22,DBLP:journals/pacmmod/DhulipalaLLM23}
as well as  recent theoretical work that
try to understand its success in practice \citep{DBLP:journals/jacm/Cohen-AddadKMM19,DBLP:conf/aistats/CharikarCNY19,DBLP:journals/jmlr/MoseleyW23,DBLP:conf/soda/CharikarCN19}.

\remove{
The mainstream hierarchical clustering methods can be  divided into two
groups, the divisive methods (top-down) and the agglomerative ones (bottom-up).
Here, we focus on the latter.
Agglomerative methods  start with a set of $n$ clusters,
corresponding to the $n$ input points, and iteratively use 
a linkage rule to merge two clusters, so after $i$ iterations we have $n-i$ clusters. 
\avglink \, is widely considered one of the most effective
 linkage methods.
Due to its relevance we can find some recent works
dedicated to improve its efficiency and scalability
\citep{DBLP:journals/pvldb/YuWGDS21,DBLP:conf/icml/DhulipalaELMS21,DBLP:conf/nips/DhulipalaELMS22,DBLP:journals/pacmmod/DhulipalaLLM23}.
}


Most of the available theoretical works give approximation bounds for \avglink \, regarding   the cost function  introduced by \citep{DBLP:conf/stoc/Dasgupta16} as
well as for some
variants of it.
Let ${\cal D}$ be the tree induced by a hierarchical clustering.
Dasgupta's cost function and its variation for dissimilarities considered in \citep{DBLP:journals/jacm/Cohen-AddadKMM19} are, respectively, given by 
\begin{equation}
\label{eq:dasgupta-def}
\dasgupta({\cal D})= \sum_{a,b \in {\cal X}} \similarity(a,b)\cdot |D(a,b)| \,\, \mbox{ and } 
\,\, \dasguptadiss({\cal D})= \sum_{a,b \in {\cal X}} \diss(a,b) \cdot |D(a,b)|,
\end{equation}
where $\similarity(a,b)$ ($\diss(a,b)$) is the similarity (dissimilarity) of points $a$ and $b$; $D(a,b)$ is the subtree of 
${\cal D}$ rooted at the least common ancestor of the leaves
corresponding to $a$ and $b$, and $|D(a,b)|$ is the number of leaves in $D(a,b)$.
In general, the existing results show
that \avglink \, achieves constant approximation for variants of Dasgupta's function
while other linkage methods do not.

However,   there is significant room for further analysis due to the following reasons.
First, Dasgupta's cost function, despite its nice properties, is less interpretable than traditional cost functions that measure compactness and separability.
Second, although the analyses based on $\dasgupta$ and its variants allow to separate \avglink \, from 
other linkage methods as 
\singlelink\, and \complink \, in terms of approximation, they do not
separate \avglink \, from a random hierarchy 
\citep{DBLP:journals/jacm/Cohen-AddadKMM19,DBLP:journals/jmlr/MoseleyW23,DBLP:conf/soda/CharikarCN19}. Moreover, for the  case in which the points lie in a metric space  
every hierarchical clustering has $1/2$ approximation for the 
maximization of \dasguptadiss \, \citep{DBLP:conf/aaai/WangM20}, so this cost function is less appealing in this relevant setting. 
Finally,  to the best of our
knowledge, \dasgupta \,  does not reveal how good are the clusters generated for a specific range of $k$. As an example, small $k$ are important
for exploratory analysis while large $k$ is important for de-duplication tasks
\citep{DBLP:conf/kdd/KobrenMKM17}.


\remove{
In particular, we show that for some natural
criteria it has an exponentially better approximation than  \complink, \singlelink \,
and a random partition for  the low $k$ regime.
\red{Improve}
}


\subsection{Our results}
Motivated by this scenario,
we  present a comprehensive study of the performance
of \avglink \, in metric spaces, with regards to  several natural criteria that
capture separability and cohesion of clustering.
In a nutshell,  these results, as explained below,  show
that average link has  much better global properties than other popular heuristics
when these two important goals are taken into account.


Let  $({\cal X},\dist)$ be a metric space, where ${\cal X}$ is a set of $n$ points.
The diameter $\diamset(S)$ of a set of points $S$ is  given by
$\diamset(S)=\max\{\dist(x,y)|x,y \in S\}$.
For a cluster $A$ and for two clusters $A$ and $B$, let  
$$ \avg(A) = \frac{1}{{|A| \choose 2}} \sum_{x,y \in A }  \dist(x,y) \,\,
\mbox{ and } \,\, \avg(A,B) = \frac{1}{|A|\cdot|B|} \sum_{x \in A } \sum_{y \in B } \dist(x,y)$$

Let $\C=(C_1,\ldots,C_k)$ be a $k$-clustering for $({\cal X},\dist)$.
To study separability we consider the average ($\costavg$) and the minimum ($\costmin$)  $\avg$  
among clusters in $\C$, that is, 
\begin{equation}
\label{eq:separability}
 \costavg(\C):= \frac{1}{{k \choose 2}}\sum_{i \ne j} \avg(C_i,C_j) 
\,\, \mbox{ and } \,\,  \costmin(\C):= \min_{i \ne j} \{ \avg(C_i,C_j) \}, 
\end{equation}
On the other hand, for studying  cohesion, we consider the maximum diameter (\maxdiamset) and the maximum  average pairwise 
distance (\maxavgset) of the clusters in $\C$. In formulae,  

\begin{equation}
\label{eq:maxdiamset-avgset}
\maxdiamset(\C):=\max\{\diamset(C_i)|1 \le i \le k\} \,\, \mbox{ and } \,\, \maxavgset(\C):=\max\{\avg(C_i)|1 \le i \le k\}
\end{equation}

We also study natural optimization goals
that capture both the separability and the cohesion of a clustering.
We define the  $\goodnessav$ and $\goodnessdm$  of a clustering $\C$
as 
\begin{equation}
\label{eq:goodness-def}
 \goodnessav(\C):=
\frac{  \maxavgset(\C)  }{\costmin(\C)} \,\, \mbox{ and } \,\,
 \goodnessdm(\C):=
\frac{ \maxdiamset(\C) }{\costmin(\C)} 
\end{equation}


Let $\A^k$ be a $k$-clustering produced by \avglink.
We first prove through a simple inductive argument that $\goodnessav(\A^k) \le 1$.
This result  does not assume that the points in ${\cal X}$ lie
in a metric space and it is tight in the sense that there are instances in which
$\goodnessav(\C) = 1$ for every $k$-clustering $\C$.
For the  related $\goodnessdm$ criterion, we present a more involved analysis which shows that $\goodnessdm(\A^k)$ as well as the approximation of \avglink \,  regarding  \OPT (the
minimum possible $\goodnessdm$)  
  are  $O ( \log n)$; these bounds are nearly  tight
since there exists an instance for which
$\goodnessdm(\A^k) $ and $\goodnessdm(\A^k) / \mbox{OPT} $  are $ \Omega ( \frac{\log n}{\log \log n})$.
Both $\goodnessav$ and $\goodnessdm$  allow an exponential separation between 
\avglink  \, and other linkage methods,  as 
\singlelink \, and \complink.
Interestingly, in contrast to \dasguptadiss\, (Eq. \ref{eq:dasgupta-def}), our criteria also allow
a very clear separation between  \avglink \, and the clustering induced by a random hierarchy.




Next, we focus on separability criteria.
Let \OPTSEPK \, be the maximum possible $\costavg$ \,
 of a $k$-clustering
for $({\cal X},\dist)$. 
We show that
$\costavg(\A^k) $ is at least  $\frac{ \OPTSEPK}{k+2\ln n }$
and that this result is nearly tight.
Furthermore, we argue that any hierarchical clustering algorithm 
that has bounded approximation regarding   $\maxdiamset$ or  $\maxavgset$ does not have  approximation
better than $1/k$ to $\costavg$.
Regarding  \singlelink \, and \complink, we
present  instances that show that their approximation with respect to $\costavg$ are exponentially worse than 
that of \avglink, for the relevant case that $k$ is small.

\remove{
 $I$ and $I'$ for which the clustering
$\S$ and $\C$, respectively,  build by these methods satisfy 
$\costavg(\S)$ is  $O(\frac{\OPTSEPK}{\sqrt{n}})$ and  
$\costavg(\C) $ is $O(\frac{\OPTSEPK}{\sqrt{n}})$.
}

We also investigate the cohesion of \avglink.
For a $k$-clustering $\C$, let \avgdiamset \, be the average
diameter of the $k$ clusters in $\C$.
Let \OPTDIMK and \OPTAVGK be, respectively, the minimum possible  \maxdiamset \,  and \avgdiamset \, of a $k$-clustering
for $({\cal X},\dist)$. 
We  prove that for all $k$, 
$\maxdiamset(\A^k)   \le \min \{k,1+ 4 \ln n\} k^{\log_2 3} \OPTAVGK $.
This result together with the instance given by Theorem 3.4 of \citep{DasguptaLaber24} allow to
separate \avglink  \, from \singlelink, in terms of
approximation,  when $k$ is $\Omega( \log^{2.41} n)$. 
We also show that
$\maxdiamset(\A^k)$ is $\Omega(k) \OPTDIMK$,
which is, to the best of our knowledge,  the first lower bound on
 the maximum diameter of \avglink.



Finally, to {\bf complement} our study, we present some experiments with 10 real datasets in which we
evaluate, to some extent, if our theoretical results line
up with what is observed in practice. These experiments
conform with our theoretical results since they also suggest that \avglink \, performs better
than other related methods when both cohesion and separability are taken
into account.

\subsection{Related work}
There is a vast literature about hierarchical agglomerative clustering
methods. 
Here, we focus on  works that provide provable guarantees for 
\avglink \, and some other well-known linkage methods.


\noindent {\bf Average-link}.
There are  works that present bounds on the approximation of  \avglink
\, regarding  some criteria \citep{DBLP:journals/jacm/Cohen-AddadKMM19,DBLP:conf/soda/CharikarCN19,DBLP:conf/aistats/CharikarCNY19,DBLP:journals/jmlr/MoseleyW23,DasguptaLaber24}.
All these works but  \citep{DasguptaLaber24} analyse
the approximation of \avglink \, 
regarding variants of 
Dasgupta's cost function.
\citep{DBLP:journals/jmlr/MoseleyW23}
assumes that the proximity between the points
in ${\cal X}$ is given by a similarity matrix.
They  show that \avglink \, is a $1/3$-approximation with
respect to the "dual" of Dasgupta's cost function.
\citep{DBLP:journals/jacm/Cohen-AddadKMM19}, as in  our work, assumes
that the proximity between points in ${\cal X}$ is 
given by a dissimilarity measure and shows
that \avglink \, has $2/3$ approximation for the problem of maximizing
\dasguptadiss\,  (Eq. \ref{eq:dasgupta-def}).
\citep{DBLP:conf/soda/CharikarCN19} show that these
 approximation ratio for
\avglink \, are tight. These papers also show that a random hierarchy obtained by a divisive heuristic that randomly splits the set of points in each cluster matches the $1/3$ and $2/3$ bounds.

\citep{DasguptaLaber24} presents an interesting approach to derive upper  bounds on cohesion criteria for a certain class
of linkage methods that includes \avglink.
They show that  $\avg(A)  \le k^{1.59} \OPTAVGK$ for every cluster $A \in \A^k$. 
Our bound on the maximum diameter of a cluster
in $\A^k$ incurs an extra factor
of $\min \{k, 1+4 \ln n \}$ to this bound and 
its proof combines their approach
with some new ideas/analyses.

 
\noindent {\bf Other Linkage Methods}.
There are also works that give bounds on the 
diameter of the clustering built by  \complink  \, and \singlelink \,
 on metric spaces \citep{DASGUPTA2005555,DBLP:journals/corr/abs-1012-3697,DBLP:conf/esa/GrosswendtR15,
 arutyunova2023upper,DasguptaLaber24}.
 Let $\C$ and ${\cal S}$ be the $k$-clustering 
 built by these methods, respectively.
  \citep{ arutyunova2023upper} shows that $\maxdiamset(\C)$  is $\Omega( k \OPTDIMK)$ 
 while 
\citep{DasguptaLaber24} shows that $\maxdiamset(\C)$  is
$O( \min \{ k^{1.30} \OPTDIMK, k^{1.59} \OPTAVGK\})$.
Regarding \singlelink,
$\maxdiamset({\cal S})$
is $\Theta( k \OPTDIMK)$ \citep{DASGUPTA2005555,arutyunova2023upper} and 
 $\Omega( k^2 \OPTAVGK)$ \citep{DasguptaLaber24}.
\citep{DBLP:journals/corr/abs-1012-3697,DBLP:conf/esa/GrosswendtR15}
give bounds for the case in which $\dist$ is the Euclidean metric.

In terms of separability criteria, it is well known that  \singlelink \,  maximizes
the minimum spacing of a clustering \citep{DBLP:books/daglib/0015106}[Chap 4.7].
Recently, \citep{LM23-Nips} observed that it also maximizes the cost of the minimum spanning tree spacing, a stronger
criterion. These criteria, in contrast to ours, just take
into account the minimum distance between points in different clusters and then they 
can be significantly impacted by noise.

\citep{DBLP:conf/soda/GrosswendtRS19}
shows that Ward's method gives a 2-approximation for $k$-means when the optimal
clusters are well-separated.

\section{Preliminaries}
 Algorithm \ref{alg:hac} shows a pseudo-code for \avglink.
The  function $\dist_{AL}(A,B)$ at line \ref{lin:dist} that measures the distance between clusters $A$ and $B$ is given by
$$\dist_{AL}(A,B):=\frac{1}{|A||B|} \sum_{ a \in A}  \sum_{ b \in B}  \dist(a,b).$$
\singlelink \, and \complink \, are obtained by replacing 
$\dist_{AL}$, in Algorithm \ref{alg:hac}, with 
$\dist_{SL}(A,B):=\min\{\dist(a,b)| (a,b)\in A \times B \}$ and 
$ \dist_{CL}(A,B):=\max\{\dist(a,b)| (a,b)\in A \times B\},$
respectively.


\small
\begin{algorithm}
\small

  \caption{{\sc H\avglink}(${\cal X}$,dist,dist$_{\L}$) }
   \begin{algorithmic}[1]

\STATE 
 ${\cal A}^{n} \gets$ clustering with $n$ unitary clusters, each one containing a point of
${\cal X}$

\STATE 
 {\bf For}  $i=n-1$ down to $1$     
 \STATE  \hspace{0.2cm} $(A,B) \gets$ clusters in ${\cal A}^{i+1}$ for which $\dist_{AL}(A,B)$ is minimum \label{lin:dist}
 \STATE  \hspace{0.2cm} ${\cal A}^{i} \gets {\cal A}^{i+1} - \{A\} - \{B\} \cup \{A \cup B\}$
   \end{algorithmic}
   \caption{Average Link}
   \label{alg:hac}
\end{algorithm}

\normalsize

A version of the triangle
inequality for averages will be employed a number of times
in our analyses. Its proof can be found in Section \ref{sec:triangle-inequality}.

\begin{proposition}[Triangle Inequality for averages]
Let $A$, $B$ and $C$ be three clusters.
Then,
$$ \avg(A,C) \le \avg(A,B)+ \avg(B,C).$$
\label{prop:triangle-ineq}
\end{proposition}

\remove{\begin{proof}
Let $a \in A$ and $c \in C$.
Then, $\dist(a,c) \le \dist(a,b) + \dist(b,c)$ for every $b \in B$.
Thus,
$$ |B| \dist(a,c) \le \sum_{b \in B } (\dist(a,b) + \dist(b,c))$$
It follows that
\begin{align*} |B|  \sum_{a \in A} \sum_{c \in C} \dist(a,c) \le   
 \sum_{a \in A} \sum_{c \in C} ( \sum_{b \in B } (\dist(a,b) + \dist(b,c))) = \\
  |C| \sum_{a \in A} \sum_{b \in B } \dist(a,b) + 
   |A| \sum_{b \in B} \sum_{c \in C } \dist(b,c) 
 \end{align*}
Dividing both sides by $|A|\cdot |B| \cdot |C|$ we establish the inequality.
\end{proof}
}

For two
disjoint clusters $A$ and $B$,  the following identity holds
$${(|A|+|B|) \choose 2} \avg(A \cup B) = { |A| \choose 2}\avg(A)+
 |A||B| \avg(A,B)+  
{ |B| \choose 2}\avg(B). $$
Dividing both sides by ${(|A|+|B|) \choose 2}$, we conclude   that $\avg(A \cup B) $ is a convex combination
of $\avg(A),\avg(B)$ and $\avg(A,B)$, a fact will be used a couple
of times in our analyses.

The following  notation
will be used throughout the text.
We use 
 $H_p=\sum_{i=1}^p \frac{1}{i}$ to denote the
$p$th harmonic number and $\A^k$ to 
refer to the $k$-clustering obtained by \avglink \, for
the instance under consideration, which will always be
clear from the context. 


\section{Cohesion and separability} 
\label{sec:separability-cohesion}
In this section, we 
 analyze the performance of 
\avglink \, with respect to both $\goodnessav$ and $\goodnessdm$ (Eq. \ref{eq:goodness-def}), criteria that simultaneously take
into account the separability and the cohesion of a clustering. Moreover, we contrast its performance with that achieved by other linkage methods.

\subsection{The $\goodnessav$ \, criterion}
\label{subsec:separability-cohesion}

We first show that $\goodnessav(\A^k) \le 1$.
The proof of this result can be found   in Section \ref{sec:thm:goodnessav},
it uses induction on the number
of iterations of \avglink \, together with a fairly simple case analysis. 

\begin{theorem}
\label{thm:goodnessav}
Let $\A^k$ be a $k$-clustering built by \avglink.
Then, for every $k$, 
 $\goodnessav(\A^k) \le 1$.
\end{theorem}

We note that the above result  does not assume the
triangle inequality and it is tight 
in the sense that for the instance  $({\cal X},\dist)$, in which the $n$ points
of ${\cal X}$ have pairwise distance 1,
 every clustering has $\goodnessav$ equal to $1$.

In Section \ref{sec:sepcoh-other-linkage},
we present instances which
show
that $\goodnessav$ can be $\Omega(n)$, $\Omega(\sqrt{n})$  and unbounded in terms of $n$
for  \singlelink, \complink\, and a random hierarchy, respectively. 
Interestingly, all the $k$-clustering, with  $2 < k \le n/2$,
induced by the hierarchical clustering obtained by these methods
satisfy these bounds. 
Furthermore, since $\goodnessdm(\C) \ge \goodnessav(\C) $
for every clustering $\C$, these bounds also hold
for the $\goodnessdm$ criterion.

A natural question that arises is whether
\avglink \, has a "good" approximation with respect to
$\goodnessav$.
Unfortunately,  the answer is no. In fact, 
 in Section \ref{sec:approx-goodnessav} we show an instance 
 where the approximation is unbounded in terms of $n$.
However, as we show in the next section, 
\avglink \, has a logarithmic  approximation with respect to
$\goodnessdm$.

\subsection{The $\goodnessdm$  criterion} 
 
We analyze the  $\goodnessdm$  of \avglink.
The results of this section  will have an important role in the analysis of both
the separability and cohesion of 
\avglink \, presented further.

First, we show that for every cluster $X$ in $\A^k$, the average distance
of a point $x \in X$ to the other points in $X-x$ is at most
a logarithmic factor of the average distance between any two clusters $Y$ and $Z$.
The proof can be found in Section \ref{sec:proof-outlier-0}.
Let $T_{i-1}$ be the cluster that contains $x$ before the $i$th merge involving $x$
and let $S_i$ be the cluster that is merged with $T_{i-1}$.
We prove by induction that $\avg(x,T_i -x) \le \ln H_{|T_i|-1} \avg(Y,Z)$, which implies
on the desired result  because $T_t=X$ for some $t$.
To establish the induction, we use the triangle inequality to write $\avg(x,T_i -x)$ as 
a function of both $\avg(x,T_{i-1} -x)$ and $\avg(T_{i-1},S_i)$,  and also argue that  $\avg(T_{i-1},S_i) \le \avg(X,Y)$.


\remove{ It is noteworthy that 
this theorem will have a key role in the analysis of the separability of
\avglink \, in the next section.
}
\remove{For \complink, consider the example presented a the end of this
section but with $k=3$ and additional 2 points $p$ and $q$ at distance $1.5\sqrt{n}$.
\complink \, builds the clustering $(\{p,q\},X\cup A,Y \cup B)$.
We have that $\diamset(\{p,q\})= 1.5 \sqrt{n}$ whilse 
$\avg(X\cup A,Y \cup B) $ is $O(1)$.
}




\begin{lemma} Let $X$, $Y$  and $Z$, with $|X| \ge 2$ and $Y \ne Z$,  be clusters of
$\A^k$.
Then, for every $x \in X$, 
we have that
$\avg(x,X) \le \avg(x,X-x) \le H_{|X|-1} \avg(Y,Z).$ 
\label{lem:outlier-0}
\end{lemma}

The next result is a simple consequence of the previous one.
 
\begin{theorem} 
\label{thorem:bounded-diam}
Let $k \ge 2$ and let $X$, $Y$  and $Z$, with $Y \ne Z$, be   clusters of a $k$-clustering
built by \avglink.
Then,
$\diamset(X) \le 2 H_{|X|-1}  \avg(Y,Z).$ 
\end{theorem}
\begin{proof}
If $|X|=1$ the result  holds  because
$\diamset(X)=0$.
Thus, we assume that $|X|>1$.
Let $x$ and $x'$ be such that 
$\dist(x,x')=\diamset(X)$.
We have that 
$$\dist(x,x') \le \avg(x,X)+\avg(X,x') \le
2 H_{|X|-1}  \avg(Y,Z) $$ 
where the first inequality follows from the triangle inequality
 and the second one due to 
Lemma \ref{lem:outlier-0}.
\end{proof} 

The next theorem shows that $\goodnessdm(\A^k) \le 2 H_n$
and that \avglink \, has a logarithmic approximation for the 
$\goodnessdm$ criterion.  
The first upper bound is a simple consequence of Theorem \ref{thorem:bounded-diam}.
Let  $\OPT$ be the minimum possible 
$\goodnessdm$.
To prove the bound on the approximation we consider two cases.
If $\OPT  \ge 1/3$  the result holds because $\goodnessdm(\A^k) \le 2\ln n \le 6 \OPT \ln n  $. If $\OPT<1/3$,
we argue that the clusters in the optimal clustering are "well separated" and,
hence, \avglink \, builds
the optimal clustering.

\begin{theorem} 
\label{thorem:cs-ratio-diam}
For all $k$, the $k$-clustering  $\A^k$ built by
\avglink \, satisfies $ \goodnessdm(\A^k) \le 2 H_n$.
Furthermore,  for all $k$, $ \goodnessdm(\A^k)$ is
$O( \log n) \cdot \OPT$
where $ \OPT$ is 
$\goodnessdm$ of the $k$-clustering with minimum possible 
$\goodnessdm$. 
\end{theorem}
\begin{proof}
The inequality $ \goodnessdm(\A^k) \le 2 H_n$ is obtained by 
using Theorem \ref{thorem:bounded-diam}, with  $X$ being  the cluster with the largest
diameter in $\A^k$  and $Y$ and $Z$ being the clusters in $\A^k$
that satisfy $\avg(Y,Z)=\costmin(\A^k)$.

Now we prove that $\A^k$ has 
logarithmic approximation.
If  $\OPT  \ge 1/3$, then
$\goodnessdm(\A^k) \le 2 H_n \le 6 \OPT  H_n$ and, hence,
  the desired result holds.

Thus, we assume $\OPT <1/3$,
Let $\C^*(k)$ be a $k$-clustering that satisfies
$\goodnessdm(\C^*(k))=\OPT$.
 The following claim will be useful.

\begin{claim}  
Let $C,C'$ be two clusters in 
$\C^*(k)$ and 
let $a,b$ be two closest points in $C$ and $C'$,
that is, $\dist(a,b)=\min \{\dist(x,y)|(x,y) \in C \times C'\}$. 
Thus,  $\dist(a,b) > \max \{ \diamset(C),\diamset(C')\}$.
\end{claim}
{\em Proof of the claim.}
We assume w.l.o.g. that $\diamset(C) \ge \diamset(C')$.  
For the sake of reaching a contradiction, assume that $\dist(a,b) \le  \diamset(C)$.
Then, it follows from the triangle inequality that 
the maximum distance between a point in $C$ and $C'$ is at
most $ 3 \diamset(C) $.
Thus, $\costmin(\C^*(k)) \le \avg(C,C') \le 3 \diamset(C)$ and so
   $\goodnessdm(\C^*(k)) \ge \diamset(C)/3 \diamset(C) =1/3$,
which contradicts our assumption. $\square.$

Now, we argue that  \avglink \, constructs the clustering $\C^*(k)$ when 
$\goodnessdm(\C^*(k)) <1/3$, so its  approximation is 1 in this case. 
For the sake of reaching a contradiction, let us assume
$\A^k \ne  \C^*(k)$. Hence, 
 at some iteration \avglink \, merges two clusters, say $A$ and $B$,
that satisfy the following properties: $A \subseteq C$ and 
  $B \subseteq C'$, where $C$ and $C'$ are two different clusters in   $\C^*(k)$.
  Let $t$ be the first iteration of \avglink \,  when it occurs.
  
Case 1)  $A \subset C$ or $B \subset C'$.
Let us assume w.l.o.g. that $A \subset C$. In this case, there
is a cluster $A'$ at the beginning of iteration $t$ such that $A' \cup A \subseteq C$. We have that
$\avg(A,A') \le \diamset(C)$ and by the above claim the minimum distance between
$A$ and $B$ is larger than $\max\{\diamset(C),\diamset(C')\}$.
Thus, $\avg(A,B) > \max\{\diamset(C),\diamset(C')\} \ge \avg(A,A')$,  
which contradicts the choice of \avglink.

Case 2) $A=C$ and $B=C'$. If $k=2$ we are done.
Otherwise, there exists a cluster $C'' \in \C^*(k)$
and two clusters $X$ and $Y$ at the beginning of iteration $t$ such that $X \cup Y \subseteq C''$. Thus, it follows from the
condition $\OPT <1/3$  that $\avg(X,Y) \le \diamset(C'') < \frac{1}{3} \costmin(\C^*(k))
\le  \frac{1}{3} \avg(C,C') \le \avg(C,C'),$ which again  contradicts the choice of \avglink.
\end{proof}

It is noteworthy that, in contrast to Theorem \ref{thm:goodnessav},
the assumption that the points lie in a metric space is necessary
to prove  Theorem \ref{thorem:cs-ratio-diam}. In Section \ref{sec:metric-space-necessary} we present an 
instance that supports this observation.


Now, we present an instance, denoted by ${\cal I}^{CS}$, that shows that the above results are nearly tight. This instance with small modifications will also be used 
to investigate the tightness of our results regarding  the separability  (Section \ref{sec:separability}) and the cohesion (Section \ref{sec:cohesion}) of
 \avglink.
We note that in most of the instances presented here,
including ${\cal I}^{CS}$, will have more than one  possible  execution for the methods
we analyze.  In these cases, we will always consider the execution that is more suitable for our purposes. These multiple executions can be avoided at the price of more
complicated descriptions that involve the addition of  small values $\epsilon$ 
to the distance or points to break ties.

Let $t$ be an integer that satisfies $t!=n$; note that $t=\Omega(\frac{\log n }{ \log \log n})$. Moreover, let
 $A_0$ be a set containing a single point located at position $p_0$
in the real line and $A_i$, for $0<i \le t-1$, be a set
of  $(i+1)!-i!$ points that are located at position $p_i$ of the real line.
We define $B_0=A_0$ and $B_i=B_{i-1} \cup A_i$, for $i \ge 1$.
Set $p_0=0, p_1=1$ and, for $i>1$, $p_i=p_{i-1}+\avg(A_{i-1},B_{i-2})
$. The set of points  for our instance ${\cal I}^{CS}$ is $B_{t-1}$
and the distance between a point in $A_i$ and a point in $A_j$ is $|p_i-p_j|$. 
The following lemma gives properties of ${\cal I}^{CS}$ and, in particular, how \avglink \, behaves on it.

\remove{
We consider $t$ groups  $A_0,A_1\ldots,A_{t-1}$, each of them containing
 points in the real line.
The only point in $A_0$ is located at position $p_0=0$,
all the points in   $A_1$ are located at position $p_i=1$ and 
all the points in $A_i$, for $i>1$, are located at position 
}


\remove{
\red{Is necesary: There is also a point $\{x\}$ located at 
$p_{t-1}+\avg(A_{t-1},B_{t-2}) - \epsilon$.?}
}

\begin{lemma}
For $i \ge 0$, we have that $|B_i|=(i+1)!$ and for
$i \ge 2$, we have $\diamset(B_{i-2})=i(i-1)/2$,
$\avg(B_{i-2},A_{i-1})=i+1$ and $p_i=i(i+1)/2$.
Furthermore, for $k \le t$, \avglink \, obtains
the $k$-clustering $\A^k=(B_{t-k}$, $A_{t-k+1},\ldots,A_{t-1})$ and, in particular,
for $k=2$ it obtains the clustering 
 $\A^2=(B_{t-2},A_{t-1})$.
\label{lem:lowerbound-cs-dm}
\end{lemma}

 From Lemma \ref{lem:lowerbound-cs-dm},  we have that
$\costmin( \A^2)=\avg(  B_{t-2},A_{t-1} )=t+1$ and 
  $\diamset(B_{t-2})=t(t-1)/2$, so 
$\goodnessdm=\frac{t(t-1)}{2(t+1)}$, which 
is $\Omega( \frac{\log n }{\log \log n}).$

Furthermore,  for the clustering 
$\A'=(A_0, B_{t-1}- A_0)$ we have
that 
\begin{equation}
\label{eq:tight-goodnessdm}
\costmin(\A')= \avg(A_0,  B_{t-1}-A_0 )  \ge 
\frac{|A_{t-1}|} {|B_{t-1}|} \avg(A_0,A_{t-1}) =
 \left ( \frac{t!-(t-1)!}{t!} \right) p_{t-1}=
\frac{(t-1)^2}{2}
\end{equation}
and $\maxdiamset(\A') \le \diamset(B_{t-1})=  (t+1)(t+2)/2$.
Thus, $\goodnessdm(\A')=O(1)$ and the logarithmic approximation
of \avglink \, to $\goodnessdm$  is also nearly tight.



\section{Separability criteria}
\label{sec:separability}
In this section, we investigate the separability of \avglink.
Recall that \OPTSEPK\, is the maximum possible $\costavg$
 of a $k$-clustering
for $({\cal X},\dist)$. 
We show that for \avglink \,  $\costavg$ is at least $\frac{\OPTSEPK}{k+2\ln n} $
and that this bound is nearly tight.
We also show that there are instances in which the
$\costavg$ of \singlelink \, and \complink \, are  exponentially smaller than that of \avglink.

\remove{
the separability of the clusters 
produced by \avglink.
There are a number of  distinct and  natural ways to measure separability among clusters.
We could measure how separated
are the two closest clusters,
how separated are the clusters in average. 
Here, we work with the average distance among
clusters, that is, $\costavg$ give by equation \ref{}.
We show that 
Finally, we show that for two more "local ways" of measuring separability,
none of the linkage methods have an approximation factor bounded by $n$.
}




Theorem \ref{thm:separability-avglink}
gives an  upper bound on $\costavg$ for \avglink \, and
its complete proof can be found in Section \ref{sec:proof-thm:separability-avglink}.
Here, we give an overview of the proof for the case $k>2$, which is
the most involved one.
The proof uses the fact established by Proposition \ref{prop:1point-approx-genera0} 
that there exists a set of $k$ points $P \subseteq {\cal X}$  
that satisfies $\avg(P) \ge \OPTSEPK$.
This holds because a set of $k$ randomly selected points that intersect
all clusters of a $k$-clustering with maximum $\costavg$  satisfies the
the desired property (in expectation). 
Having this result in hands, it is enough to show that 
$\avg(P)$ is $O(( k+H_{n-1} )\costavg(\A^k))$.

This bound on $\avg(P)$ is obtained by relating the distance of each pair of points  $p,p' \in P$ with the average distance between clusters in $\A^k$.
Let $p,p' \in P$ and let $A$ and $A'$ be clusters in $\A^k$ such
that $p \in A$ and $p' \in A'$.  Moreover, let $S$ be a cluster in $\A^k$, with $S \notin \{A,A'\}$. From the triangle inequality 
we have that  $\dist(p,p')=\avg(p,p') \le \avg(p,A) +
\avg(A,S)+ \avg(S,A')+ \avg(A',p')$. Then, 
by bounding both $\avg(p,A)$ and $\avg(A',p')$ via    Lemma \ref{lem:outlier-0},
with  $Y$ and $Z$ satisfying $\avg(Y,Z) \le \costavg(\A^k)$,
we conclude that
$\dist(p,p') \le 2 H_n \costavg(\A^k) +
\avg(A,S)+ \avg(S,A').$
In general lines,  the result is then 	established
by   averaging this inequality for all $S \notin \{A,A'\}$ and for all $p, p' \in P$.

\medskip

\remove{
\red{
We shall note that  
the existence of a "non-natural" clustering that has constant approximation
for $\costavg$   may raise some concern about the benefit 
of optimizing this criterion. Our understanding is that
it is desirable for a clustering having a a high value for  $\costavg$,
but   this criterion should not be used isolated.
}
}

\begin{proposition}
There is a set of points $P \subseteq {\cal X}$ with the following properties:
$|P|=k$ and 
$\avg(P) \ge  \OPTSEPK$.
\label{prop:1point-approx-genera0}
\end{proposition}

\begin{theorem}
For every $k$, the $k$-clustering $\A^k$ obtained by \avglink \,
satisfies $\costavg(\A^k) \ge  \frac{\OPTSEPK}{ k+ 2H_n }.$
\label{thm:separability-avglink}
\end{theorem}


 We present two instances that, together, show that
 the previous theorem is nearly tight.
The first is the instance ${\cal I}^{CS}$   presented right after Theorem \ref{thorem:cs-ratio-diam}.
For ${\cal I}^{CS}$,   
the clustering $\A^2=( A_{t-1} , B_{t-2})$ built by \avglink \, satisfies
$\costavg(\A^2)=\avg( A_{t-1} , B_{t-2})=t+1.$
On the other hand,  Eq. (\ref{eq:tight-goodnessdm}) shows that
$\costavg(\A') =\frac{(t-1)^2}{2}$, for the clustering  $\A'=(A_0, B_{t-1}- A_0)$. Thus, for ${\cal I}^{CS}$,  $\costavg(\A^2)$ is  
  $O(\frac{\OPTSEPK \log \log n }{\log n })$.

Now, we present our second instance, denoted by ${\cal I}^{sep}_k$. 
Let $k$ be an odd number and let $D$ and $\epsilon$ be positive numbers.
The set of points of ${\cal I}^{sep}_k$ is given by
$S_1 \cup S_2 \cup S_3$, 
where  $|S_1|=|S_2|=(k-1)/2$ and 
 $S_3=\{s_i| 1 \le i \le k-2 \}$.
We have $\dist(x,y)=\epsilon$ for $x,y \in S_1$,
$\dist(x,y)=\epsilon$ for $x,y \in S_2$, 
$\dist(x,y)=1$ for $x,y \in S_3$ and 
$\dist(x,y)=D$ if $x$ and $y$ are not in the same
set.  

For ${\cal I}^{sep}_k$, when $D$ is sufficiently large and  $\epsilon$ is
sufficiently small, 
$\A^k=(S_1,S_2,s_1,\ldots,s_{k-2})$ and
$\costavg(\A^k)=O(D/k)$.
On the other hand,  the $\costavg$ of the $k-$clustering that has the cluster $S_3$ and $k-1$ singletons
	corresponding to the points in $S_1 \cup S_2$ is 
 $\Omega(D)$.
Thus, $\costavg(\A^k)$ is $O(\OPTSEPK/k)$. 




We note that \singlelink \, and \complink \, also obtain the $k$-clustering
$\A^k$ for ${\cal I}^{sep}_k$, so the upper bound 
$\OPTSEPK/k$ also holds for them.
In Section \ref{sec:costavg-other-methods} we
present instances that show that $\costavg$ is 
$O(\frac{\OPTSEPK}{\sqrt{ n}})$ for both \singlelink \, and \complink.

The instance ${\cal I}^{sep}_k$
 is particularly interesting because it  also 
shows that natural cohesion and separability criteria  
can be conflicting. The key reason is that any method $M$ 
with bounded approximation (in terms of $n$) regarding 
\maxdiamset \, or to \maxavgset \, (Equation \ref{eq:maxdiamset-avgset})  has to build the $k$-clustering $\A^k$ for 
${\cal I}^{sep}_k$. Thus,  by analysing ${\cal I}^{sep}_k$
we can conclude that the approximation factor of $M$ to $\costavg$ is $O(1/k)$
  and to $\costmin$ is  $O(1/D)$.
The details can be found  in 
 Section \ref{sec:conflicting}.

\section{On the cohesion of average-link}
\label{sec:cohesion}

In this section, we prove that $\maxdiamset(\A^k) \le \min \{ k, 1+4\ln n \} k^{1.59} \OPTAVGK$
and  we also present an instance which shows
that $\maxdiamset(\A^k) \ge k \OPTDIMK$.

\cite{DasguptaLaber24} presented an interesting  approach to devise upper  bounds
on cohesion criteria for a class
of linkage methods that includes \avglink. 
Although this approach was used to show that the maximum
pairwise average distance of a cluster in $\A^k$ is at most $k^{1.59} \OPTAVGK$, 
it cannot be employed, at least directly, to bound
the maximum diameter of a cluster in $\A^k$.
Thus, to obtain our $(1+4\ln n )k^{1.59} \OPTAVGK$ bound we combine the results
of \citep{DasguptaLaber24} with Theorem \ref{thorem:cs-ratio-diam} while 
for  the $ k^{1+1.59} \OPTAVGK $  bound we add some new ideas/analysis
on  top of those from \citep{DasguptaLaber24}.



The analysis in \cite{DasguptaLaber24} keeps a dynamic partition of the clusters produced by 
the linkage method under consideration.
Each group in the partition is a set of clusters denoted by {\em family}. A point $p$ belongs to a family $F$ if it belongs to some cluster in $F$. Thus,  $\diamset(F)$  is given by the maximum distance among the points that belong to $F$. The approach bounds
the diameter of each family $F$
  as (essentially) a function of 
the clusters that $F$ touches in a  target $k$-clustering $\T=(T_1,\ldots,T_k)$.
The bound on $\diamset(F)$ is then used to upper bound the diameter of the clusters in $F$. 
For a $k$-clustering $\C$, let $\avgdiamset(\C):=\frac{1}{k} \sum_{i=i}^k \diamset(C_i).$
As in \cite{DasguptaLaber24}, we use as the target clustering the one with minimum \avgdiamset.

We explain how the families evolve along the execution of
a linkage method, in particular \avglink. 
 Initially, we have $k$ families, $F_1,\ldots,F_k$, where  $F_i$ is a family that contains $|T_i|$ clusters, each one being a point from
$T_i$. Furthermore,  the families are organized in a directed forest $D$
that  initially consists of $k$ isolated nodes,
where the $i$th node corresponds to family $F_i$.

We specify how the families and the forest $D$  are updated when
the linkage method merges the clusters  $g$ and $g'$ belonging
to the families $F$ and $F'$, respectively. Assume w.l.o.g. $|F| \ge |F'|$.
We have the following cases:
\begin{itemize}
\item[case 1]  $|F'|=1$ and $|F|>1$.  In this case two new
families are created, $F^{new}:= F-\{g\}$ and  $F^{new'}:=\{g \cup g'\}$.
Moreover, $F^{new}$ and $F^{new'}$  become, respectively, parents of $F$ and $F'$ in $D$   

\item[case 2]  $|F'|>1$ or $|F|=1$. In this case, only one 
family is created, $F^{new} := (F \cup F' \cup \{g \cup g'\})-g -g'$.
Moreover, $F^{new}$ becomes parent of both $F$ and $F'$ in $D$.
\end{itemize}

We say that a family $F$ is \emph{regular} if $|F|>1$.

\begin{proposition}[Proposition 3.1 of \cite{DasguptaLaber24}] 
 At the beginning of each iteration of \avglink  \,
 at least one of the roots of the forest $D$ corresponds to a regular family. 
\label{prop:reg-family-number}
\end{proposition}

Let ${\cal M}$ be  the class of
linkage methods (Algorithm \ref{alg:hac}) whose function $f$,  employed to measure the distance between clusters $A$ and $B$
satisfies 
\begin{equation}
 \{ \dist(a,b) | (a,b) \in A \times B\} \leq f(A,B) \le \diamset(A \cup B)
 \label{eq:dasg-laber-condition}
\end{equation}

\begin{proposition} [Proposition 5.1 of \cite{DasguptaLaber24} ]
The diameter of every regular family $F$ produced along
the execution of a linkage method in ${\cal M}$
 is   at most   $
k^{\log_2 3} \OPTAVGK$. 
 \label{prop:diameter-expansion}
\end{proposition} 

Note that the function $\dist_{AL}$ employed by
\avglink \, satisfies the condition given by  (\ref{eq:dasg-laber-condition})
and, thus, the above proposition holds for \avglink. 

We are ready to establish the main result of this section.
\remove{ If $S$ belongs to a regular family then
 it follows from Proposition \ref{prop:diameter-expansion} that
$ \diamset(S) \le  k^{\log_2 3} \OPTAVGK$.
Thus, we assume that $S$ does not belong to a regular family. 
}

\begin{theorem}
Every  cluster $S$ in $\A^k$ satisfies
 $ \diamset(S) \le  \min\{k,4 \ln n +1\} k^{\log_2 3} \OPTAVGK.$
\label{thm:diam-bound}
\end{theorem}
\begin{proof}
Let $V=\{T \in \T| S \cap T \ne \emptyset\}$
be the set of clusters of the target
clustering $\T$ that intersect $S$.
We build a graph $G$ whose nodes correspond to the clusters in $V$.
At the beginning of \avglink's execution, $G$  contains the set of nodes $V$
and no edges.

At each iteration, there are two possibilities for the  clusters $g$ and $g'$
that are merged by \avglink:   $(g \cup g' )\cap S = \emptyset$ 
or $(g \cup g') \subseteq S $.
We define how  $G$ is updated in  each case:

Case 1)  $(g\cup g') \cap S = \emptyset$.
In this case,  $G$ is not updated.

Case 2)  $(g \cup g' )\subseteq S $.
Let $x$ and $y$ be points in $g$  and $g'$ such
that $\dist(x,y)$ is minimum and 
let $T^x$ and $T^y$ be the clusters in $\T$ that contain
$x$ and $y$, respectively. 
We add an edge of weight $\dist(x,y)$ between  $T^x$ and $T^y$.
We say, in this case, that $x$ and $y$  are {\em associated} with the edge
that links $T^x$ to $T^y$.

We need the following two claims:

\begin{claim} $\dist(x,y)  \le k^{\log_2 3} \OPTAVGK$.
\end{claim}
\noindent {\em Proof of the claim.}
Let $H$ be a regular family
  at the beginning of iteration $t$ Such family does exist
due to Proposition \ref{prop:reg-family-number}.
Moreover, let $h$ and $h'$ be two clusters in $H$.
We have that 
$$\dist(x,y) \le \dist_{AL}(g,g') \le \dist_{AL}(h,h') \le \diamset(h \cup h')
\le \diamset(H) \le k^{\log_2 3} \OPTAVGK,$$
where the second inequality holds
by the choice of \avglink \, and the  last  inequality holds due to the Proposition 
\ref{prop:diameter-expansion}. $ \,\, \square$

\begin{claim}
For a cluster $C$, let $V_{C}:=\{T \in \T | T \cap C \ne \emptyset \}$.
Let $S'$ be a cluster generated by \avglink \, that is a subset of $S$. 
Then,  when $S'$ is created, the subgraph of $G$ induced by $V_{S'}$ is connected. 
\end{claim}
\noindent {\em Proof of the claim}
If $|S'|=1$ the property holds.
Let $S'$ be a cluster obtained by merging $S_1$ and $S_2$.
By induction, the property holds for $S_1$ and $S_2$.  
Since an edge is added between nodes in $V_{S_1}$ and $V_{S_2}$  then
the property also holds for $S$.  $ \square$

\remove{$V_{g \cup g'}$ is also connected.     
We note that the set $V_{g \cup g'}$
is  connected in $G$, that is, there is 
a path between any two nodes in $V_{g \cup g'}$.
To see that assume by induction that $V_g$ and $V_{g'}$ are connected. Since
an edge is added between nodes in $V_{g'}$ and $V_{g'}$  then 
$V_{g \cup g'}$ is also connected.
}

Thus, at the end of the algorithm, $G$ is connected and 
each of its edges
has weight at most $k^{\log_2 3} \OPTAVGK$.
Let $x$ and $y$ be points in $S$ such that $\dist(x,y)=\diamset(S)$
and
let $T^x=v_1 \ldots v_\ell=T^y$ be a path in $G$
from $T^x$ to $T^y$.

Consider a sequence of points $x=p_1p'_1 \ldots p_\ell p'_\ell=y$,
where  $ p_i$ and $p'_i$ are the points in $v_i$ associated with
the edge $v_{i-1}v_i$ and $v_iv_{i+1}$, respectively. 
From the triangle inequality
\begin{align*}
\dist(x,y) \le \sum_{i=1}^{\ell-1} \dist(p'_i,p_{i+1})+
\sum_{i=1}^{\ell} \dist(p_i,p'_i) \le 
  (k-1)  k^{\log_2 3} \OPTAVGK + \sum_{i=1}^k \diamset(T_i) \le \\
  (k-1)  k^{\log_2 3} \OPTAVGK + k\OPTAVGK 
\end{align*}
\remove{
It follows from  this observation and the
triangle inequality that
   $$\diamset(S) \le k\OPTAVGK + (k-1)  k^{\log_2 3} \OPTAVGK.$$
}

For the logarithmic bound, let 
$S_1$ and $S_2$ be the two clusters that are merged to form $S$. 
At the beginning of the iteration in which  $S_1$ and $S_2$ are merged, Proposition \ref{prop:reg-family-number}
 assures 
that there exists a regular family, say $H$. 
Let $h$ and $h'$ be two clusters in $H$. By Proposition \ref{prop:diameter-expansion},
$\avg(h,h') \le \diamset(H) \le k^{\log_2 3} \OPTAVGK$. Thus, by Theorem \ref{thorem:bounded-diam},
$\diamset(S_1) \le 2 \ln n \cdot  \avg(h,h') \le 2 \ln n \cdot  k^{\log_2 3} \OPTAVGK$
and
$\diamset(S_2) \le 2 \ln n \cdot  k^{\log_2 3} \OPTAVGK.$
Let $s_1 \in S_1$ and $s_2 \in S_2$ be such
that $\dist(s_1,s_2)=\min \{\dist(p,q) | (p,q) \in S_1 \times S_2)$.
Since $S_1$ and $S_2$ are merged
we have that $\dist(s_1,s_2) \le \avg(S_1,S_2) \le \avg(h,h')\le
k^{\log_2 3} \OPTAVGK.$
Thus, 
$\diamset(S)  \le  \diamset(S_1) + \dist(s_1,s_2) +  \diamset(S_1) \le
   (1+4\ln n ) k^{\log_2 3} \OPTAVGK.$
 \end{proof}

Theorem 3.4 of \cite{DasguptaLaber24} presents 
an instance with $n=2k-2$ points for which
\singlelink \, builds a $k$-clustering that
has a cluster whose diameter  is  $\Omega(k^2 \OPTAVGK)$.
Thus, this result together with Theorem \ref{thm:diam-bound} show a separation between 
\avglink \, and \singlelink \, when 
$k$ is $\Omega(\log^{2.41}n)$.

Our last theoretical result is a lower bound on the maximum diameter
of the clustering built by \avglink. 
Its proof can be found in the Section 
\ref{sec:daim-lowbound} and it  employs an augmented version
of instance ${\cal I}^{CS}$, presented right after Theorem \ref{thorem:cs-ratio-diam}.

\begin{theorem}
There is an instance
for which the $k$-clustering $\A^k$ built by \avglink \,
satisfies $\maxdiamset(\A^k) \in  \Omega(k \OPTDIMK)$
\label{thm:diam-lowbound}
\end{theorem}


\section{Experiments}
\label{sec:experiments}
In this final section, we briefly present an experiment in which  
we evaluate whether \avglink, in addition to having better theoretical bounds,
it also has a better performance in practice for the studied criteria.
We employed
10 datasets  
 and used the Euclidean metric to measure distances.
For each of them, we executed \avglink, \complink\, and \singlelink,
for the following  sets of values of $k$:
{\tt Small}=$\{k| 2 \le k \le 10\}$,
{\tt Medium}=$\{k|\sqrt{n}-4 \le k \le \sqrt{n}+4\}$
and {\tt Large}=$\{k|k=n/i \mbox{ and }  2 \le i \le 10 \}$.
More details, as well as the results of our experiment with other distances, can be
found in Section \ref{sec:experiments-extra}.

Table \ref{tab:experiments} shows 
the average ratio between the result of a method
and that of the best one, grouped by criterion and set of $k$.
Each entry is the average of 90 ratios (9 $k'$s and 10 datasets) and each of these
ratios for a method $M$ is a value between 0 and 1 that is obtained by
dividing the minimum between the result of $M$ and that of the best method 
by the maximum between them. The letters A, C and S are the initials
of the evaluated methods.

\remove{the number of pairs (dataset,$k$) that each
method was at least as good as the other two for the  criteria
defined in Equations (\ref{eq:separability}),
(\ref{eq:maxdiamset-avgset}) and
(\ref{eq:goodness-def}).
}

\begin{table}
\begin{small}
\label{tab:experiments}
\caption{Average ratio between the result of a method
and the best one for each criterion and each group of $k$. The best
results are bold-faced}
\begin{center}
\begin{tabular}{lccc||ccc||ccc}
       & \multicolumn{3}{c}{Small} & \multicolumn{3}{c}{Medium} & \multicolumn{3}{c}{Large} \\ \hline
       & A      & C      & S      & A       & C       & S      & A       & C      & S      \\ \hline
$\costmin$    & \textbf{0,99} & 0,82          & 0,76 & \textbf{1}    & 0,81          & 0,68       & \textbf{1}    & 0,81          & 0,72       \\
$\costavg$    & \textbf{0,97} & 0,82          & 0,94 & 0,97          & 0,9           & \textbf{1} & 0,98          & 0,96          & \textbf{1} \\
$\maxdiamset$ & 0,85          & \textbf{1}    & 0,72 & 0,8           & \textbf{1}    & 0,48       & 0,76          & \textbf{1}    & 0,38       \\
$\maxavgset$  & 0,95          & \textbf{0,96} & 0,86 & \textbf{0,99} & 0,89          & 0,71       & \textbf{0,99} & 0,84          & 0,67       \\
$\goodnessdm$ & \textbf{0,96} & 0,92          & 0,63 & 0,95          & \textbf{0,97} & 0,4        & 0,93          & \textbf{0,99} & 0,33       \\
$\goodnessav$ & \textbf{0,98} & 0,82          & 0,69 & \textbf{1}    & 0,73          & 0,51   & \textbf{1}    & 0,68          & 0,4
\end{tabular}
\end{center}
\end{small}
\end{table}

Concerning separability criteria, \singlelink\, and  \avglink \,
have the best results for $\costavg$. The latter has some advantage when $k$ is small, which is in line with its better worst-case bound for small $k$ (results from Section \ref{sec:separability}).
For $\costmin$, \avglink\, has a huge advantage, which is
not surprising since its linkage rule tries to increase
$\costmin$ at each step by merging the 
the clusters $A$ and $B$ for which $\avg(A,B)=\costmin(\C)$,
where $\C$ is the current clustering.

Regarding cohesion criteria, \complink \, and \avglink \,  were the best methods.
They had close results for  
$\maxavgset$ while for  $\maxdiamset$ the former had a strong dominance.
These results align with ours and those 
 from \citep{DasguptaLaber24}, in the sense that they show that
these linkage methods present better worst-case
upper bounds than \singlelink \, when the comparison is made against 
$\OPTAVGK$.  Moreover, the advantage of \complink \, for  $\maxdiamset$ is also
expected since it is the "natural" greedy rule to minimize the maximum diameter 
(See Proposition 2.1 of \cite{DasguptaLaber24}).
 


For $\goodnessdm$, \avglink \,  and \complink\  present the best results,
with the former being slightly superior for the small $k$ and the latter
being slightly superior when  $k$ is not small.
\avglink \,  has a huge dominance for the $\goodnessav$ criterion, which lines
up with the theoretical results from Section \ref{subsec:separability-cohesion}.

In summary, these experiments, together with our theoretical results,
provide evidence that  \avglink\, is a better choice  when both cohesion and
separability are relevant.

\noindent {\bf Acknowledgements}
The work of the first author is partially supported by  CNPq (grant 310741/2021-1).
This study was financed in part by the Coordenação de Aperfeiçoamento de Pessoal de Nível Superior - Brasil (CAPES) - Finance Code 001

\noindent {\bf Limitations.}
We have not identified a major limitation in our work. That said, the assumption that the points lie in a metric space used in our results (except Theorem \ref{thm:goodnessav})  could be seen as a limitation. On the experimental side, having more than 10 datasets would give our conclusions more robustness.

\remove{
\section{Limitations}
We do not identify a major limitation in our work. That said, the assumption that the points lie in a metric space used in our results (except Theorem \ref{thm:goodnessav}) and the fact that our bounds are not tight (most are nearly tight), could be seen as limitations. Moreover, having more than 10 datasets would give our conclusions more robustness.
}

\remove{

The following patterns were observed. 
\avglink\, outperformed the other methods in terms 
$\maxavgset$ and $\costmin$.  As a consequence, 
it was also better for the the $\goodnessav$, which is aligned with our theoretical results (Theorem \ref{}).  The result of \avglink\, for $\costmin$
is not surprising since it is a natural greedy approach for minimizing
this metric.
\complink\, outperformed the other methods in the
$\maxdiamset$ criterion for almost all instance. Again, this is not surprising since 
the linkage rule employed by \complink, as proved in \cite{}, is equivalent to the linkage rule that at each step merge two clusters $A$ and $B$ for which $\diamset(A \cup B)$
is minimum. $\complink$ was also better than \avglink\, in the $\goodnessdm$
criterion in $XX\%$ of the cases, despite the better approximation factor achieved by the latter. Finally, for $\costavg$, \singlelink \, obtained the best results, followed by \avglink. 
   }

\bibliography{biblio}
\bibliographystyle{unsrtnat}

\newpage
\appendix

\section{Proof of proposition \ref{prop:triangle-ineq}}
\label{sec:triangle-inequality}
\begin{proof}
Let $a \in A$ and $c \in C$.
Then, $\dist(a,c) \le \dist(a,b) + \dist(b,c)$ for every $b \in B$.
Thus,
$$ |B| \dist(a,c) \le \sum_{b \in B } (\dist(a,b) + \dist(b,c))$$
It follows that
\begin{align*} |B|  \sum_{a \in A} \sum_{c \in C} \dist(a,c) \le   
 \sum_{a \in A} \sum_{c \in C} ( \sum_{b \in B } (\dist(a,b) + \dist(b,c))) = \\
  |C| \sum_{a \in A} \sum_{b \in B } \dist(a,b) + 
   |A| \sum_{b \in B} \sum_{c \in C } \dist(b,c) 
 \end{align*}
Dividing both sides by $|A|\cdot |B| \cdot |C|$ we establish the inequality.
\end{proof}

\section{Proofs of section \ref{sec:separability-cohesion}}

\subsection{Proof of Theorem \ref{thm:goodnessav}}
\label{sec:thm:goodnessav}
\begin{proof}
When $k=n$ the result is valid because $\avg(A^n)=0$ for every $A \in \A^n$.
We assume by induction that the result holds for $k+1$ and we prove
that it also holds for $k$.
Let  $A$ and  $B$ be the clusters in
$\A^{k+1}$ that are merged to obtain $\A^k$, so $\A^{k}=\A^{k+1} \cup (A \cup B) - \{A,B\}$.
Let $S,T$ and $U$ be clusters in $\A^{k}$, with $T \ne U$.
It is enough to prove that $\avg(S) \le \avg(T,U)$.

Case 1) $A \cup B \notin \{S,T,U\}$.
In this case, $S,T,U \in \A^{k+1}$ and, then, 
 by induction, $\avg(S) \le \avg(T,U)$.

Case 2) $ A \cup B=S$ and $S \notin \{T,U\}$.
Since $A,B,T,U \in \A^{k+1}$,
the induction hypothesis assures that $\avg(A) \le \avg(T,U)$ and
$\avg(B) \le \avg(T,U)$ and the \avglink \, rule
ensures that $\avg(A,B) \le \avg(T,U)$.
Since $\avg(S)$ is a convex combination of
$\avg(A),\avg(B)$ and $\avg(A,B)$,
the above inequalities imply that $\avg(S)=\avg(A \cup B) \le \avg(T,U)$.

Case 3)  $A \cup B=S$ and $S \in \{T,U\}$.
We assume w.l.o.g. that $S=T$. 
The induction hypothesis and the \avglink \, rule 
guarantee that $\max\{\avg(A),\avg(B),\avg(A,B)\} \le \min\{\avg(A,U),  \avg(B,U)\} $
Since $\avg(S,U)$  is a convex combination
of $\avg(A,U)$ and  $\avg(B,U)$
and $\avg(S)$  is a convex combination of
$\avg(A),\avg(B)$ and $\avg(A,B)$,
the above inequality  implies that $\avg(S)=\avg(A \cup B) \le \avg(T,U)$.

Case 4) $S \ne  A \cup B$ and $A \cup B \in \{T,U\}$.  
We assume  w.l.og. that $ T= A \cup B $.
Since $S,A,B,U \in\C^{k+1}$, the induction hypothesis
assures that $\avg(S) \le \min \{ \avg(A,U), \avg(B,U)\}$
Since  $\avg(T,U)$ 
is a convex combination
of $\avg(A,U)$ and  $\avg(B,U)$,
the above inequality assures that
 $\avg(S) \le \avg(T,U)$.
\end{proof}

\subsection{Lower bounds on $\goodnessav$ for other  methods}
\label{sec:sepcoh-other-linkage}
The following examples show that  the $\goodnessav$ of \complink, \singlelink \, and a random hierarchy
can be much higher than that of \avglink \,
in metric spaces.

\noindent {\bf single-linkage.}
Consider the instance with
$n$ points $x_1,\ldots,x_n$ in the real line, where
$x_i=1$, if $i=1$, and 
  $x_i=x_{i-1}+1-i \epsilon$, for $i>1$.
For  $\epsilon$ sufficiently small, \singlelink \, builds the $k$-clustering 
$\C=(x_1,x_2,\ldots,x_{k-1},\{x_k,\ldots,x_n\})$.
We have that $\avg(\{x_k,\ldots,x_n\})$ is $\Omega(n-k)$ while $\avg(x_1,x_2)=1-\epsilon$,
so that $\goodnessav(\C)$ is $\Omega(n-k)$.

\medskip

\noindent {\bf complete-linkage.}
Let $t=2^{m}-1$, where $m$ is a positive integer and let
$p=2(t^2+t)$.
We build an instance
whose set of points ${\cal X}=A \cup B \cup C \cup D \cup E$ has 
$n=2p$ points,
where  $A,B,C,D$ and $E$ are sets of points in $  \mathbb{R}^{p+1}$
that satisfy the following properties:
\begin{itemize}
\item the first coordinate of the points in $A \cup B \cup C \cup D$ 
is the only one that has a value different than 0;
\item $A=\{a_1,\ldots,a_t\}$ and the first coordinate of $a_i$ is equal to $i+1/2$;
\item $B=\{b_1,\ldots,b_t\}$ and the first coordinate of $b_i$ is equal to  $-(i+1/2)$;
 \item $C$ has $t^2$ points and all have the first coordinate $1/2$;
\item $D$ has $t^2$ points and all have the first coordinate $-1/2$;
\item $E=\{e_1,\ldots,e_p\}$, where 
the value of the first coordinate of $e_i$ is $t^2$, the $(i+1)$th coordinate
has value 
$1.5t$ and all other coordinates have value equal to 0. 
\end{itemize}

The distance between any two points in ${\cal X}$ is given by the $\ell_1$ metric.
Hence, the distance between any two points in $E$ is $3t$,
the distance between points in $A \cup B \cup C \cup D$ is at most 
$2t+1$ and the distance between a point in $A \cup B \cup C \cup D$
and a point in $E$ is at least $t^2$.
For $i \le p$, let $E_i=\{e_i,\ldots,e_p\}$.

Thus, for $2 < k < p=n/2$, there is a way to  break ties for which
the $k$-clustering obtained by 
\complink  \, is $\C^k=(A \cup C, B \cup D,e_1,e_2,\ldots,e_{k-3},
E_{k-2})$.

We have that $\max\{\dist(a,d) \in A \times D\} \le t+1$,
 $\max\{\dist(b,c) \in B \times C\} \le t+1$ and
  $\max\{\dist(a,b) \in A\times B\} \le 2t+1$. 
Thus, we get that 
\begin{align*} 
\costmin(\C^k) \le \avg((A \cup C,B \cup D)) \le \\
\frac{1}{(t^2+t)^2}\left( \sum_{x \in A }\sum_{y \in B } \dist(x,y) + \sum_{x \in A }\sum_{y \in D  } \dist(x,y)+\sum_{x \in C }\sum_{y \in B } \dist(x,y) + \sum_{x \in C }\sum_{y \in D}\dist(x,y) \right)\\
\le \frac{ t^2(2t+1)+t^3(t+1)+t^3(t+1)+t^4}{(t^2+t)^2} \le 3
\end{align*}

Since 
$ \maxavgset(\C) \ge \avg(E_{k-2})=3t$, we get that
$\goodnessav(\C^k)$ is $\Omega(t)$ and, hence, $\Omega(\sqrt{n})$.

\medskip

{\bf random hierarchy.}
To analyze a random hierarchy, we first
need to define how it is generated.
We start with a random permutation of  the points
in ${\cal X}$ and a clustering $\C$ containing initially
the cluster comprised by all points in   
${\cal X}$.
Let $x_1,\ldots,x_n$ be the points in ${\cal X}$ according
to the order given by the permutation.
Then, we perform the following steps until we have $n$ clusters:

\begin{itemize}
\item $j \gets$ a randomly selected a number in the set $\{1,2,\ldots,n-1\}$. 
\item If the
points $x_j$ and $x_{j+1}$ are in the same cluster $C \in \C$
\begin{itemize}
\item split  $C$
into $C_{\le} =\{ x_i \in C | i \le j  \}$ and the cluster $C_{>}=C-C_{\le}$. 
\item Update $\C$ by replacing $C$ with $C_{\le}$ and $C_{>}$
\end{itemize}

\end{itemize}

After $t$ splits we have a clustering with $n-t$ clusters.

Now, we 
consider an instance with $n$ points
and 3 groups $X$, $Y$ and $Z$,
that satisfy $|X|=|Y|=(n-1)/2$ and  $Z=\{z\}$.
The distance between any two points in $X$ 
is $1$ and the same holds for $Y$.
Moreover, the distance between points in $X$ and $Y$ is $2$.
The distance of $z$  to any other point is  $D >>n^2 $. 
Any $k$-clustering, with $k \ge 3$,
has  $\costmin \le 2$ because at least two clusters do not contain $z$.
Let $k \le n/2$.
The probability that $z$ is a singleton in the $k$-clustering when $z \notin \{x_1,x_n\}$
is 
$$\frac{{n-3 \choose k-3}}{  {n-1 \choose k-1}}=\frac{(k-1)(k-2)}{(n-1)(n-2)} <\frac{1}{4}$$ Then, with probability at least $3/4$, there will be a cluster $C$
that contains $z$ 
and a point in $ X \cup Y$, which implies that
$E[\avg(C)] \ge D/4n^2$. 
Thus, with probability at least
$3/4$ the $k$-clustering induced by the random hierarchy has $\costavg$  $\Omega(D/4n^2)$, 
when $z \notin \{x_1,x_n\}$. 
Since the probability of $z \notin \{x_1,x_n\}$
is $(n-2)/n$, the same bound holds when we drop this constraint.


\subsection{On the approximation of \avglink \, for $\goodnessav$} 
\label{sec:approx-goodnessav}

Let $n$ be an even number, $k=2$ and $\epsilon$ a positive number very close to 0.
Consider 4 set of points $S_1$, $S_2$, $S_3$ and $S_4$,
where $S_1=\{s_1\}$,$S_2=\{s_2\}$ and
$S_3$ and $S_4$ have $n/2-1$ points each. 
We have $\dist(x,y)=\epsilon$ for $x,y \in S_3$, 
$\dist(x,y)=\epsilon$ for $x,y \in S_4$,
  $\dist(s_1,s_2)=T$
and  $\dist(x,y)=T$ for $(x,y) \in S_3 \times S_4$.
In addition, we have $\dist(s_1,x)=2T$ for $x \ne  s_2$
and $\dist(s_2,y)=2T$ for $y \ne  s_1$.

Clearly, the  $4$-clustering obtained by \avglink \, is $(S_1,S_2,S_3,S_4)$.
Then, to obtain a 2-clustering,  it merges the clusters $S_1$ and $S_2$ and, next,
 $S_3$ and $S_4$, so that the final 2-clustering
is $\A^2=(S_1 \cup S_2, S_3 \cup S_4)$, which satisfies 
$\maxavgset(\A^2)=T$ and $\costmin(\A)=2T$.
On the other hand, for the clustering 
$\S=(S_1 \cup S_3, S_2 \cup S_4)$, we have
that  $\maxavgset(\S)$ is $O( T/n^2)$ and $\costmin(\S) \ge T$.
Thus, the approximation of \avglink \, is $\Omega(n^2)$

\subsection{Triangle inequality is necessary for Theorem \ref{thorem:cs-ratio-diam}}

\label{sec:metric-space-necessary}
We present an 
instance that shows that the assumption that points
lie in a metric space is necessary to establish Theorem \ref{thorem:cs-ratio-diam}.

Let $A$ and $B$ be  sets with $n/2-1$ and $n/2$ points, respectively.
We have  
$\dist(a,a')=1$ if $a,a' \in A$; $\dist(b,b')=1$ if $b,b' \in B$ 
and $\dist(a,b)=4$ if $(a,b) \in A \times B$.
Moreover, let $p$ be a point that is not in $A \cup B$.
There is a point
$a \in A$ for which $\dist(a,p)=n/2-2$ and  for all other points $a' \in A -\{a\}$, $\dist(a',p)=2$. Moreover,  $\dist(p,b)=4$ for $b \in B$.

For this instance \avglink \,  builds
the $2$-clustering $\A^2=(A \cup \{p\},B)$.  
We have that $\diamset(A \cup p)=n/2-2$ and $\avg(A \cup p ,B)=4$,
Thus, $\goodnessdm(\A^2)$ is $\Omega(n)$.
On the other hand, for the clustering $\A'=(A , B \cup p )$,
$\goodnessdm(\A')$ is $O(1)$, so the approximation of \avglink \, is $\Omega(n)$.

\subsection{Proof  of Lemma \ref{lem:outlier-0}}
\label{sec:proof-outlier-0}
\begin{proof}
The first inequality holds because $\avg(x,X)=\frac{|X|-1}{|X|}
\avg(x,X-x)$.  Thus, we just need to prove the second one.

Let $S_1$ be the first cluster merged with $x$ by \avglink \,  and let
$S_i$, for $i>1$, be the cluster merged with 
$S_1 \cup \cdots \cup S_{i-1}$ by \avglink \,. Define
$T_0:=\{x\}$ and, for $i \ge 1$, 
 $T_i:=T_{i-1} \cup S_i$.

Furthermore, define $e_i$ and $m_i$ as $e_i:=\avg(T_{i-1},S_{i})$ and $m_i :=\avg(x,T_i-x)$
, respectively.
Note that there is $t$ for which $T_t=X$ and, hence,  
$m_t=\avg(x,X-x)$.

We have that 
\begin{align}
m_{i+1}=  \frac{|T_i|-1}{|T_{i+1}|-1} \avg(x,T_i-x)  +\frac{|S_{i+1}|}{|T_{i+1}|-1} \avg(x,S_{i+1}) \le
\label{eq:m-evol} \\
\frac{|T_i|-1}{|T_{i+1}|-1} m_i  +\frac{|S_{i+1}|}{|T_{i+1}|-1} (m_i+ e_{i+1})
= m_i + \frac{|S_{i+1}|}{|T_{i+1}|-1} e_{i+1},
\label{eq:m-evol0}
\end{align}
where the inequality follows from the triangle inequality.

Let us consider the beginning of the iteration in which  $T_{i-1}$ and $S_i$ are
merged. At this point 
we have $\ell \ge 1$  clusters
$Y_1,\ldots,Y_{\ell}$ such that $Y=Y_1 \cup \cdots \cup Y_{\ell}$
and  $\ell'$ clusters $Z_1,\ldots,Z_{\ell'}$ such that
$Z=Z_1 \cup \cdots \cup Z_{\ell}$.
Note  
that there exist $i$ and  $j$ such that 
 $\avg(Y_i,Z_j) \le \avg(Y,Z)$.
 Thus, we must have $e_i \le \avg(Y,Z)$, otherwise
 \avglink \, would merge $Y_i$ and $Z_j$ 
 rather than $T_{i-1}$ and $S_i$.

To establish the result, we show by induction that 
$m_{i} \le  \avg(Y,Z) \cdot H_{|T_{i}|-1},$ for $i \ge 1$.
The lemma is then established by taking $i=t$, where $t$ satisfies
$T_t=X$.

For $i=1$, we have
$m_{1}=e_1 \le  \avg(Y,Z) <  \avg(Y,Z) \cdot H_{|T_{1}|-1}$.
We assume by induction that 
$m_{i-1} \le  \avg(Y,Z) \cdot H_{|T_{i-1}|-1}$.
By inequality  ( \ref{eq:m-evol})-(\ref{eq:m-evol0}), 

$$m_{i} \le m_{i-1} + e_i \frac{|S_{i}|}{|T_{i}|-1}  \le 
\avg(Y,Z)\left ( \sum_{h=1}^{|T_{i-1}|-1}  \frac{1}{h} \right )  + \avg(Y,Z)\left (  \sum_{h=|T_{i-1}|}^{|T_{i}|-1} \frac{1}{h}  \right ) =  \avg(Y,Z) \cdot H_{|T_i|-1}
$$ 
\end{proof}

\section{Proof of Lemma \ref{lem:lowerbound-cs-dm}}

\begin{proof}
First, we note that   
$$|B_{i-1}| =\sum_{h=0}^{i-1} |A_i|= \sum_{h=0}^{i-1} (h+1)!-h!= i!,$$
for $i \ge 1$.

Moreover, for $i \ge 2$, we have that 
\begin{align} \avg(A_i,B_{i-1})= \frac{|A_{i-1}|}{|B_{i-1}|} \avg(A_i,A_{i-1})+
\frac{|B_{i-2}|}{|B_{i-1}|} 
 \avg(A_i,B_{i-2})=  \label{eq:motononic0-14Feb}
\\
\frac{|A_{i-1}|}{|B_{i-1}|}  \avg(A_i,A_{i-1})+
\frac{|B_{i-2}|}{|B_{i-1}|}  (\avg(A_{i},A_{i-1}) + \avg(A_{i-1},B_{i-2}) ) = \\
 \avg(A_i,A_{i-1})+
\frac{|B_{i-2}|}{|B_{i-1}|}  \avg(A_{i-1},B_{i-2}) =\\
\left (1 +  \frac{|B_{i-2}|}{|B_{i-1}|}  \right ) \avg(A_{i-1},B_{i-2}) ,
\label{eq:motononic5-14Feb}
\end{align}
where the last identity follows because
$\avg(A_i,A_{i-1})=p_i-p_{i-1}=\avg(A_{i-1},B_{i-2})$.

By applying the above  equation successively, we conclude that
$$\avg(A_i,B_{i-1})=(i+1) \cdot \avg(A_1,B_0)=(i+1)$$
and, hence, $$p_i=1+\sum_{h=1}^{i-1} (h+1) = \frac{i(i+1)}{2}.$$
Thus, 
$$\diamset(B_{i-1})=p_{i-1}-p_0=p_{i-1}=\frac{i(i-1)}{2}$$

Now we show that
at the  beginning of the  step $(n-t)+i$
\avglink \,  keeps a 
clustering that contains the cluster $B_{i-1}$ and
the clusters $A_j$, for $ i \leq j \leq t-1$.
First, we observe that after $n-t$ steps
\avglink \, produces a $t$-clustering  
$(A_0,\ldots,A_{t-1})$ since points in
the same group $A_i$ are located at the same position.
We analyze what happens in the remaining $t-1$ steps.

For $i=1$ the result holds because $B_0=A_0$.
We assume as an induction hypothesis that 
at beginning of the  step $(n-t)+i$, we have
the clusters $B_{i-1}$ and $A_j$, for $j \ge i$.
By construction, for $i\le r < s$, 
$$\avg(A_{s},A_r)=p_s-p_r >p_{i+1}-p_i=\avg(A_i,B_{i-1}),$$
Moreover,
$$ i-1= \avg(A_i, B_{i-1}) < \avg(A_j,B_{i-1}),$$
for $j>i$.
Thus, 
$\avglink$ \, prefers merging $A_i$ and $B_{i-1}$ rather than
any other pair of clusters, which completes the inductive step.
\end{proof}

\section{Proofs from section \ref{sec:separability}}

\subsection{Proof of Proposition \ref{prop:1point-approx-genera0}}

\begin{proof}

Let $\C^*=(C^*_1,\ldots,C^*_k)$ be a $k$-clustering that maximizes $\costavg$.
Let ${\cal Q}$ be the family of 
sets of points $Q$ such that $|Q|=k$ and $Q$  intersects all clusters $C^*_1,\ldots,C^*_{k}$.
Let $P=\{p_1, \ldots,p_{k}\}$ be a set in  ${\cal Q}$ that 
satisfies
$  \avg(P) \ge \avg(Q),$
for every $Q \in {\cal Q}$.
Moreover, let $U=\{u_1,\ldots,u_{k}\}$ be a set of $k$ points
where $u_i$ is  randomly selected  from $C^*_i$.
It follows from the choice of $P$  that
\begin{align*}
\frac{k(k-1)}{2} \avg(P) \ge \frac{k(k-1)}{2} E[ \avg(U)]  
= \\ 
E \left [ \sum_{i=1}^{k-1}\sum_{j=i+1}^{k} \dist(u_i,u_j) \right ] =
\sum_{i=1}^{k-1}\sum_{j=i+1}^{k} E \left [ \dist(u_i,u_j) \right ] =
 \sum_{i=1}^{k-1}\sum_{j=i+1}^{k} \avg(C^*_i,C^*_j) \ge \\
\frac{k(k-1)}{2} \costavg(C^*) 
\label{eq:choice-Ci2}
\end{align*}
\end{proof}

\subsection{Proof of  Theorem \ref{thm:separability-avglink}}
\label{sec:proof-thm:separability-avglink}
\begin{proof}

Let $P=\{p_i| 1 \le i \le k \}$
be the $k$ points given by  Proposition 
\ref{prop:1point-approx-genera0} and  
let   $h$ be a function  that maps 
each point $p \in P$ into its  cluster
in $ \A^k$. 
Moreover, let  $Y$ and $Z$ be clusters in $\A^k$ that satisfy 
$\avg(Y,Z)=\costmin(\A^k)$.

Let $p$ and $p'$ be distinct points in $P$. We consider two cases:

\medskip

Case 1) $p$ and $p'$  belong to the same cluster $A$ in $\A^k$.
From Theorem \ref{thorem:bounded-diam} we have that 
$$\dist(p,p') \le \diamset(A) \le 2 H_{ |A|} \avg(Y,Z) =2 H_{ |A|}\costmin(\A^k) $$
Thus,
\begin{equation}
\label{eq:thm-sep-case1-old}
\sum_{p,p' \in P \cap A} \dist(p,p') \le 
\sum_{p,p' \in P \cap A}  2 H_{ |A|} \costmin(\A^k).
\end{equation}

By considering all clusters $A \in \A^k$ we get 

\begin{equation}
\label{eq:thm-sep-case1}
\sum_{p,p' \in P \atop h(p)=h(p') } \dist(p,p') \le
\sum_{p,p' \in P \atop h(p)=h(p') }
2 H_{n} \costmin(\A^k)
\end{equation}

\medskip

Case 2) $p$ and $p'$  belong, respectively, to different clusters $A$ and $A'$ in $\A^k$.
We consider two subcases:

{\it subcase 2.1)} $k=2$. In this case,  from the triangle inequality, we have that 
$\dist(p,p')=\avg(p,p') \le \avg(p,A)+ \avg(A,A') + \avg(A',p').$
By using Lemma \ref{lem:outlier-0}, we have that $\avg(p,A) \le H_{n-1} \avg(A,A')=
H_{n-1} \costmin(\A^k)$
and $\avg(p',A') \le H_{n-1} \avg(A,A')=
H_{n-1} \costmin(\A^k)$.
Thus,
\begin{equation}
\label{eq:casek2}
 \sum_{p, p' \in P \atop h(p)\ne h(p')}   \dist(p,p')  = \dist(p,p') \le 
  2 H_{n-1} \costmin(\A^k) + \avg(A,A'),
\end{equation}
where the first identity holds because $P=\{p,p'\}$.

{\it  subcase 2.2)} $k>2$. 
Let $S$ be a cluster in $\A^k - \{A,A'\}$.
From the triangle inequality, we have that
$$\dist(p,p') = \avg(p,p') \le \avg(p,A) + \avg(A,S) + \avg(S,A') + \avg(A',p')$$
If $|A| = 1$,  $\avg(p,A)=0 \le H_{|A|} \cdot \costmin(\A^k) $.
Moreover, if $|A| \ge 2$, it follows from Lemma \ref{lem:outlier-0}  that $\avg(p,A) 
\le    H_{|A|} \cdot \avg( Y,Z)=H_{|A|}\costmin(\A^k)$. 
Analogously, 
we have $\avg(p',A') \le  H_{|A'|} \costmin(\A^k)$.
Thus,
$$\dist(p,p') \le H_{|A|}  \costmin(\A^k) + \avg(A,S) + \avg(S,A') + 
H_{|A'|}\costmin(\A^k).$$
By averaging over all possible $S \in \A^k - \{A,A'\}$ we get that

$$ 
\dist(p,p') \le 
\cdot  2 H_n  \costmin(\A^k)+ \frac{1}{k-2} \sum_{S \notin \{A,A'\}} (\avg(A,S) + \avg(S,A')) 
$$

By adding over all points $p \in P \cap A$ and $p' \in P \cap A'$ we get that

\begin{align*} 
 \sum_{p \in P \cap A} \sum_{p' \in A' \cap Y} \dist(p,p') \le \\
 \sum_{p \in P \cap A} \sum_{p' \in P \cap A'} 2 H_n  \costmin(\A^k) +
 \frac{|P \cap A| \cdot  |P \cap A'| }{k-2} \sum_{S \notin \{A,A'\}} (\avg(A,S) + \avg(S,A'))
\end{align*}

By adding the above inequalities for $p,p' \in P$, with $h(p) \ne h(p')$,
we get that
\begin{align} 
 \sum_{p, p' \in P \atop h(p)\ne h(p')}   \dist(p,p') \le \label{eq:Apri29-1}\\
  \sum_{p,p' \in P \atop h(p)\ne h(p')} 2 H_n \cdot \costmin(\A^k) +
\frac{1}{k-2}\sum_{A,A' \in \A^k \atop A \ne A'} |P \cap A| \cdot  |P \cap A'|  \sum_{S \notin \{A,A'\}} (\avg(A,S) + \avg(S,A') = \\
  \sum_{p,p' \in P \atop h(p)\ne h(p')}2 H_n \cdot \costmin(\A^k) +
\frac{1}{k-2}  \sum_{A,A' \in \A^k \atop A \ne A'} 
  (|P \cap (A \cup A')|) \cdot (k- |P \cap (A \cup A')|)  \cdot  \avg(A,A') \le \\
 \sum_{p,p' \in P \atop h(p)\ne h(p')} 2 H_n \cdot  \costmin(\A^k) +
 k \sum_{A,A' \in \A^k \atop A \ne A'} \avg(A,A'), \label{eq:Apri29-3}
\end{align}
where the last inequality holds because 
$(|P \cap (A \cup A')|) \cdot (k- |P \cap (A \cup A')|)  \le k^2 /4 .$


If we compare  inequalities (\ref{eq:Apri29-1})-(\ref{eq:Apri29-3})  with 
inequality (\ref{eq:casek2}), we conclude that (\ref{eq:Apri29-1})-(\ref{eq:Apri29-3})  also hold for the subscase $k=2$.

Then, by adding inequality (\ref{eq:thm-sep-case1}) with the  inequalities (\ref{eq:Apri29-1})-(\ref{eq:Apri29-3})
and also using the fact $\costmin(\A^k) \le \costavg(\A^k)$,
we get
that
$$ \sum_{p, p' \in P \atop p \ne p'} \dist(p,p') \le 
 2H_n \frac{k(k-1)}{2} \costmin(\A^k) +
 k \sum_{A,A' \in \A^k \atop A \ne A' } \avg(A,A')  \le (2 H_n+k) \frac{k(k-1)\costavg(\A^k)}{2} $$
Proposition \ref{prop:1point-approx-genera0}
ensures that 

$$ \frac{k(k-1)}{2} \OPTSEPK  \le  \frac{k(k-1)}{2}  \avg(P)=
 \sum_{p, p' \in P} \dist(p,p') $$
Thus, from the two previous inequalities, we conclude that 
$$ \costavg(\A^k) \ge   \frac{\OPTSEPK}{2 H_n+k}.$$



\remove{
$$ \OPT \le  2 \costavg(\C') \le \frac{k}{k-1} \sum_{p,p' \in P } \avg(p,p) \le
 O(\log n)  \sum_{x,y \in P } \avg(f(x),f(y)) +
 k \sum_{X,Y \in \A^k \atop X \ne Y} \avg(X,Y)
 \le (\log n + k) \costavg(\A^k)
$$}
\end{proof}

\subsection{The $\costavg$ criterion for other linkage methods}
\label{sec:costavg-other-methods}

The following  instances show that the separability
of both  \singlelink \, and \complink \, can be much lower than
$\frac{\OPTSEPK}{\log n}$.

For  \singlelink, 
consider the instance ${\cal X}=A \cup B \cup \{p\}$, where $A$ contains $n-1-\sqrt{n}$ points and $B$ contains $\sqrt{n}$ points
$b_1,\ldots,b_{\sqrt{n}}$. Moreover, we have $\dist(x,y)=\epsilon$, for $x,y \in A$,
$\dist(b_i,x)=i$ for every point $x \in A$ and $\dist(b_i,b_j)=|i-j|$.
Moreover,  $\dist(p,x)=1+\epsilon$,
for every point $x \in A$.
and $\dist(p,b_i)=1+\epsilon+i$
In this case, \singlelink \, builds the clustering $(A \cup B,\{p\})$.
We have that
$ \costavg(A \cup B, p) \le 2$,
while $\costavg (A \cup p, B)$ is $ \Omega(\sqrt{n})$.


Regarding \complink, we consider the instance  presented
at  Section \ref{sec:sepcoh-other-linkage}, but without the
set $E$,
that is, the set of points is  ${\cal X}=A \cup B \cup C \cup D$.
When $k=2$, \complink \, builds the 
clustering $(A \cup C, B \cup D)$ that has
$\costavg$  $O(1)$ while the 
clustering 
$(A,C \cup D \cup B)$ satisfies
$$\costavg(A,C \cup D \cup B) \ge \frac{ \frac{t^2}{2} (2t^2+t)}{ t (2t^2+t))}=\frac{t}{2}.$$
Since $t=\Theta(\sqrt{n})$,  we conclude
that the separability of \complink \, for this instance is $O( \frac{\OPTSEPK}{\sqrt{n}})$.

\subsection{Separability and cohesion can be conflicting}
\label{sec:conflicting}

Recall that for instance ${\cal I}^{sep}_k$
\avglink \,  builds the $k$-clustering $\A^k=(S_1,S_2,s_1,s_2,\ldots,s_{k-2})$.
Note that  $\maxdiamset(\A^k)=\maxavgset(\A^k)=\epsilon$.  
Let $\A'$ be 
a $k$-clustering different from $\A^k$.
We argue that  $\maxdiamset(\A') \ge 1$ and $\maxavgset(\A')$ is $\Omega( 1/k^2)$.
In fact, if $\A'$ has a
cluster $A$ that satisfies $|A| \ge 2$ 
and $|A \cap S_3| \ge 1$, then
$\maxdiamset(\A') \ge 1$ and $\maxavgset(\A')$ is $\Omega( 1/k^2)$.
Otherwise, if $\A'$ does not have such a cluster, 
then all points in $S_3$ must
be singletons in $\A'$.
Since $\A' \ne \A^k$,
there is a cluster in $\A'$ that contains both a point in $S_1$ and
a point in $S_2$.  Thus,
$\maxdiamset(\A') = D$ and $\maxavgset(\A')$ is $\Omega( D/k^2)$.

Let ${\cal M}$ be the class of methods with bounded  approximation regarding  
$\maxdiamset$ or to $\maxavgset$.
Then any method $M \in {\cal M}$  builds the clustering $\A^k$.
Since  $\costavg(\A^k)$ is $O(D/k)$ and  there is a $k$-clustering for ${\cal I}^{sep}_k$  whose
$\costavg$ is $\Omega(D)$,
we conclude that
the approximation factor of any method $M \in {\cal M}$ regarding  $\costavg$ is $O(1/k)$. 

Now, we consider $\costmin$.
We have that $\costmin(\A^k)=1$.
Let   
${\cal B}=(B_1,\ldots,B_{k})$ be a $k$-clustering
with the following properties:
(i) $|B_i \cap S_3| \ge 1$ for each $i \le k-2$;
(ii) each  
$B_i$, with $i \ge 2$, has exactly one point in $S_1 \cup S_2$
(iii)  $B_{k-1}$ has a point in $S_1$ and 
$B_{k}$ has a point in $S_2$.
We have that  $\costmin({\cal B})$ is  $\Omega(D)$.
Thus,  
any  method $M \in {\cal M}$ has
approximation $O(1/D)$  to $\costmin$, that is,
the approximation is unbounded in terms of $n$.

\section{Proof of Theorem \ref{thm:diam-lowbound}}
\label{sec:daim-lowbound}

\begin{proof}
Let ${\cal I}$ be the instance obtained by
augmenting the  instance ${\cal I}^{CS}$, presented right after Theorem \ref{thorem:cs-ratio-diam},
with the  points $x_0, \ldots,x_{t-1}$,
where $\dist(x_i,A_i)=t+1 +\epsilon$ and for $i \ne j$,
$\dist(x_i,x_j)=|p_j-p_i|+2(t+1+\epsilon)$ 
 and $\dist(x_i,A_j)=|p_j-p_i|+t+1+ \epsilon$.

Consider $t=k$.
We argue   that the
$(k+1)$-clustering obtained by \avglink \, for ${\cal I}$ consists
of the clusters $(B_{k-1},\{x_0\},\ldots,\{x_{k-1}\})$.
In fact, in its  first steps \avglink \, obtains
the $2k$-clustering $(A_0,\ldots,A_{k-1},x_{0},\ldots,x_{k-1})$
since the distance between points in $A_i$ is 0.
In the next $k-1$ steps, \avglink \, does  not make a merge involving a point $x_i$ because the average
distance of $x_i$ to any other cluster is larger $k+1$ and,
by Lemma \ref{lem:lowerbound-cs-dm},
the average distance between $B_{i-2}$ and $A_{i}$
is $i+1 \le k+1$. Thus, the execution of \avglink \,
for ${\cal I}$ 
merges the same clusters that are merged in the instance	
${\cal I}^{CS}$  and, then, ends up  with the
$(k+1)$-clustering $(B_{k-1},\{x_0\},\ldots,\{x_{k-1}\}).$
 
Thus, for instance ${\cal I}$, the maximum diameter of a cluster in $\A^k$ 
is at least $\diamset(B_{k-1})$, which is $\Omega(k^2)$, while the 	$k$-clustering $(x_0 \cup A_0, \ldots,x_{k-1} \cup A_{k-1})$
has diameter $k+ \epsilon$.
\end{proof}

\section{Experiments: extra details}
\label{sec:experiments-extra}

Table \ref{tab:datasets} presents our datasets with their main characteristics.

\begin{table}[h]
\begin{center}
\caption{Datasets}
\label{tab:datasets}
\begin{tabular}{c|c|c|c}
Dataset & $n$  & $d$ & Source\\ \hline
Airfoil & 1501  & 5 & \cite{misc_airfoil_self-noise_291} \\
Banknote & 1371 & 5 &  \cite{misc_banknote_authentication_267}\\
Collins  & 1000 & 19 & OpenML \\
Concrete & 1028 & 8 & \cite{misc_concrete_compressive_strength_165} \\
Digits  & 1797 & 64 &  \cite{Alpaydin1998}\\
Geographical Music & 1057 & 116 & \cite{misc_geographical_origin_of_music_315} \\
Mice & 552 & 77 & \cite{misc_mice_protein_expression_342} \\
Qsarfish & 906 & 10 & \cite{misc_qsar_fish_toxicity_504}\\
Tripdvisor & 979 & 10 & \cite{misc_travel_reviews_484}\\
Vowel  & 990 & 10 & UCI  \\
\end{tabular}
\end{center}
\end{table}

\begin{figure}[ht]
\vskip 0.2in
\begin{center}
\centerline{\includegraphics[width=0.9\columnwidth]{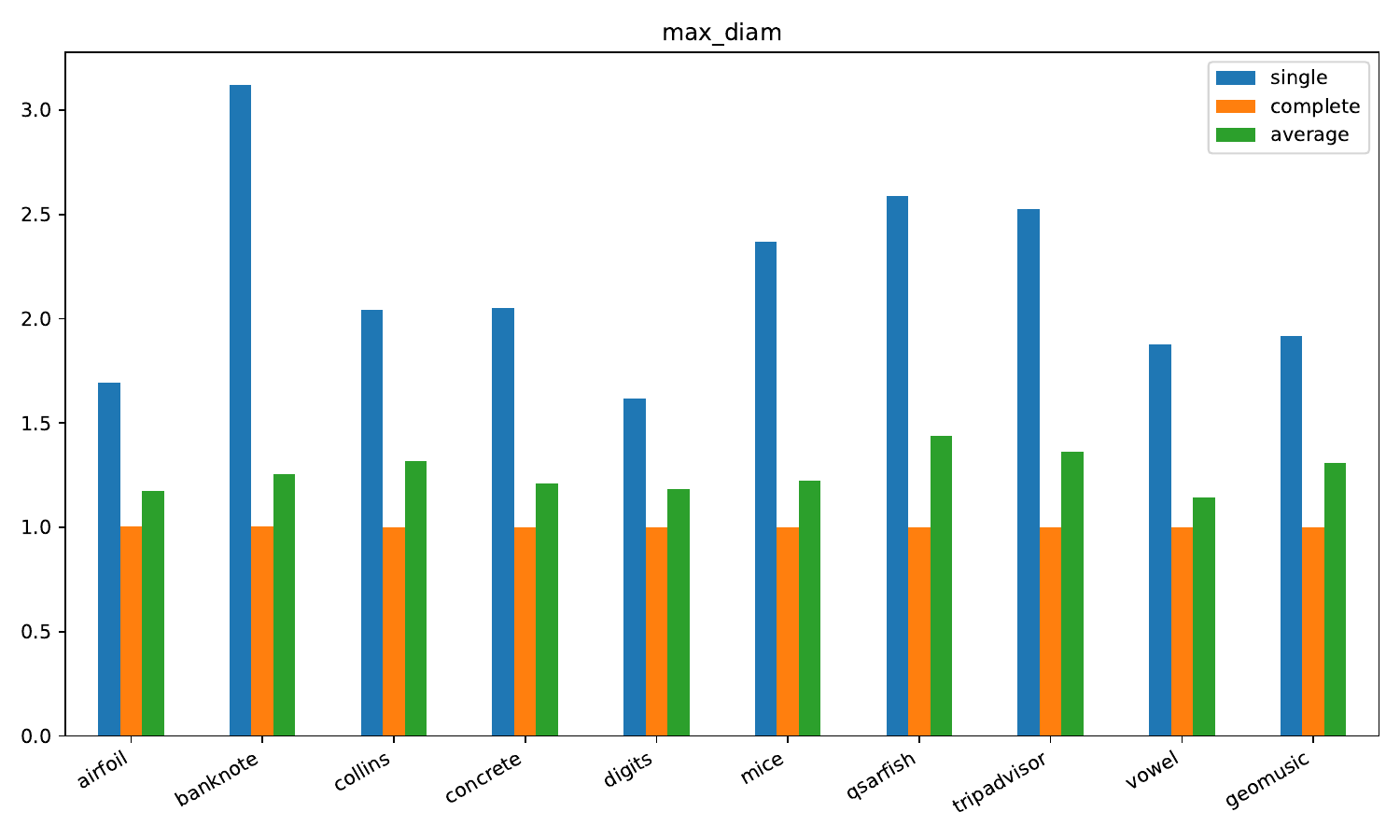}}
\caption{Results for the \maxdiamset \, for the different datasets. For interpreting the bars, the lower  the better}
\label{fig:max_diam}
\end{center}
\vskip -0.2in
\end{figure}

\begin{figure}[ht]
\vskip 0.2in
\begin{center}
\centerline{\includegraphics[width=0.9\columnwidth]{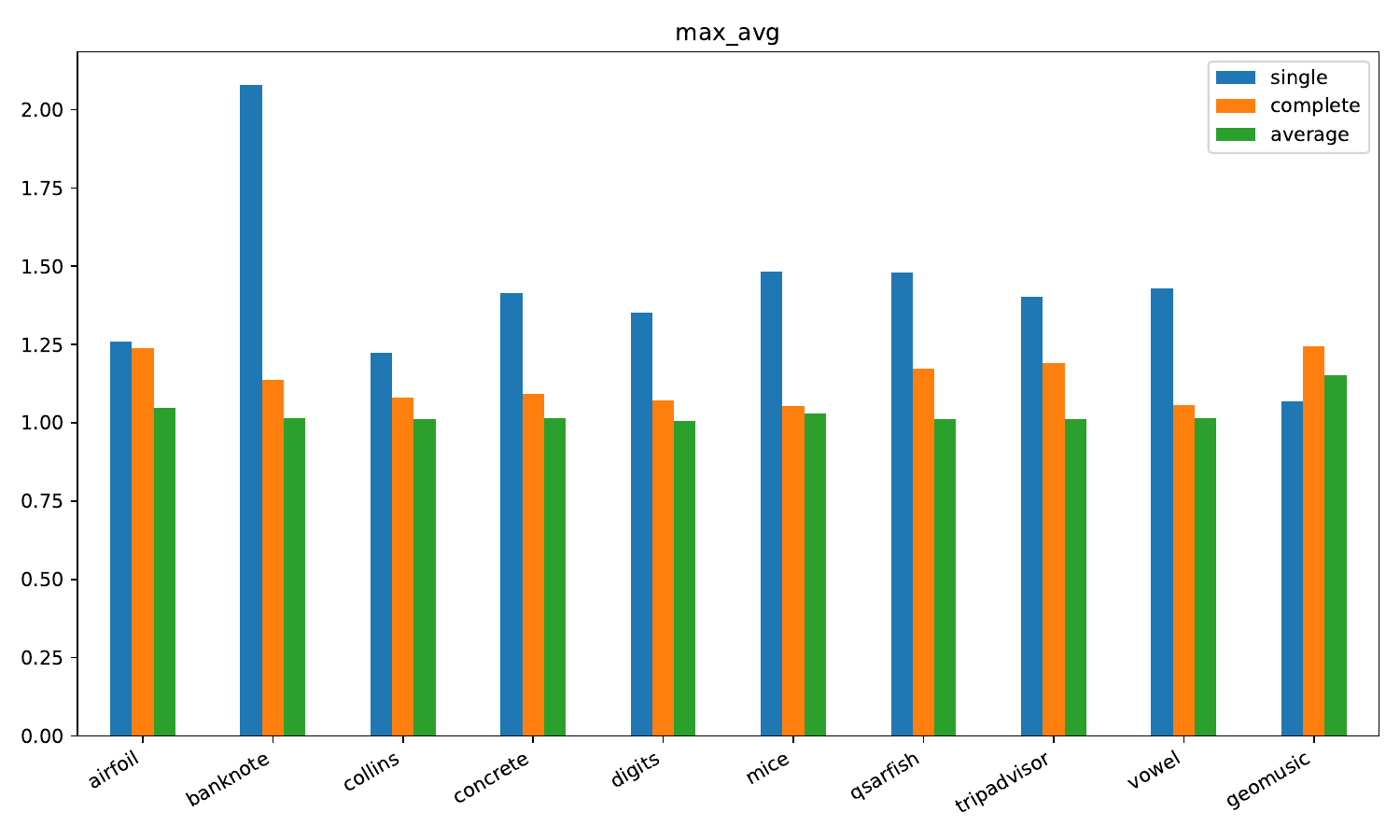}}
\caption{Results for the \maxavgset \, for the different datasets. For interpreting the bars, the lower  the better}
\label{fig:max_avg}
\end{center}
\vskip -0.2in
\end{figure}

\begin{figure}[ht]
\vskip 0.2in
\begin{center}
\centerline{\includegraphics[width=0.9\columnwidth]{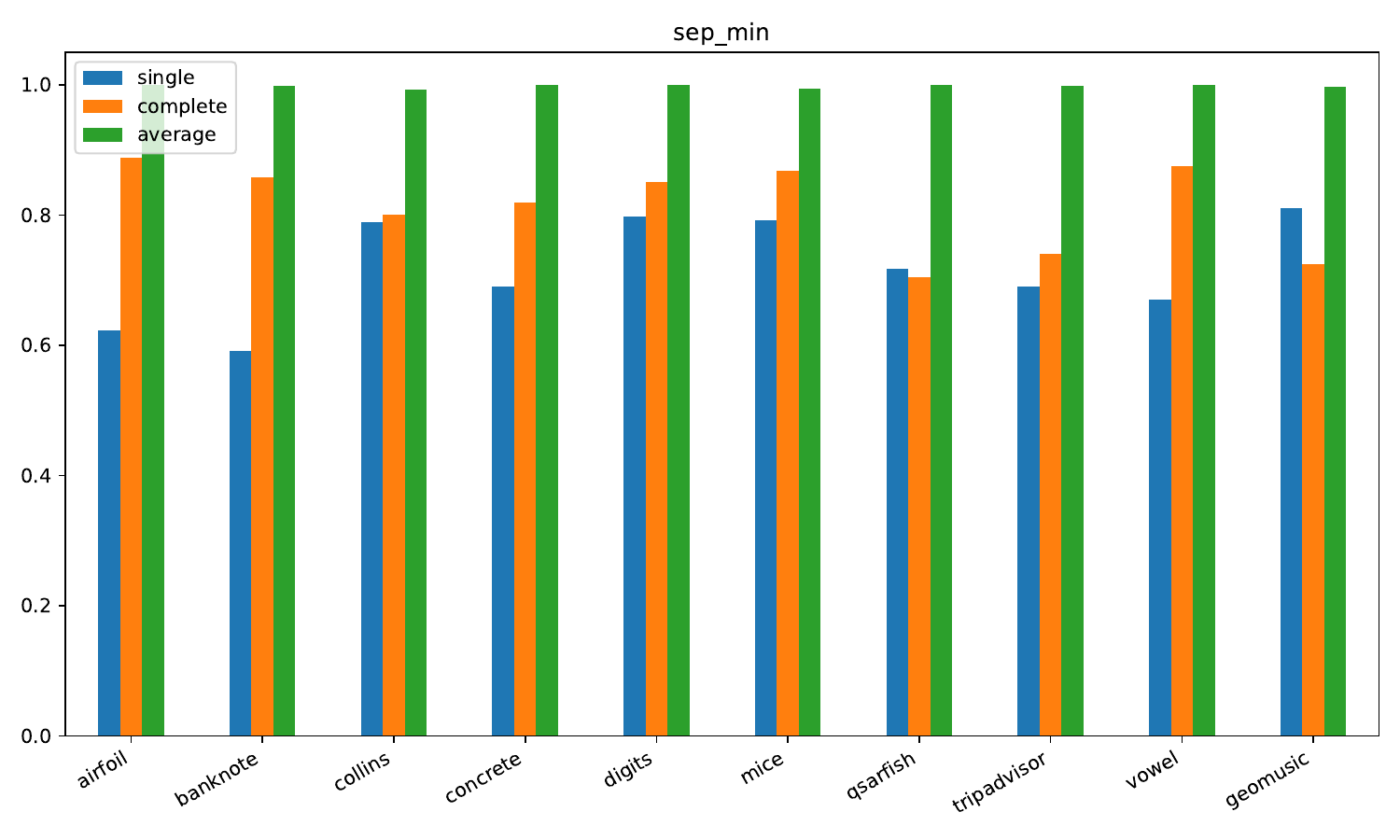}}
\caption{Results for the $\costmin$ \, for the different datasets. For interpreting the bars, the higher  the better}
\label{fig:sep_min}
\end{center}
\vskip -0.2in
\end{figure}

\begin{figure}[ht]
\vskip 0.2in
\begin{center}
\centerline{\includegraphics[width=0.9\columnwidth]{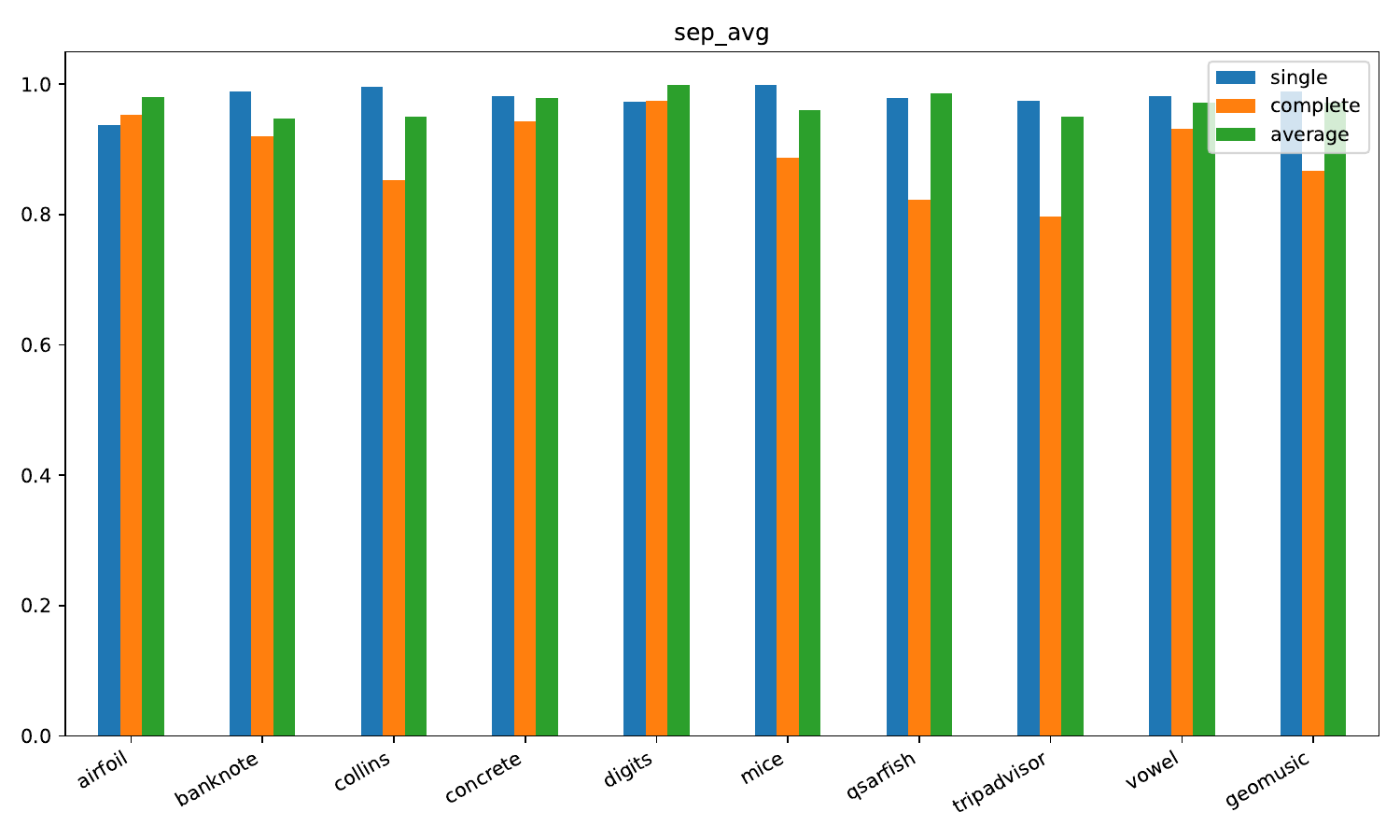}}
\caption{Results for the $\costavg$ \, for the different datasets. For interpreting the bars, the higher  the better}
\label{fig:sep_AV}
\end{center}
\vskip -0.2in
\end{figure}

\begin{figure}[ht]
\vskip 0.2in
\begin{center}
\centerline{\includegraphics[width=0.9\columnwidth]{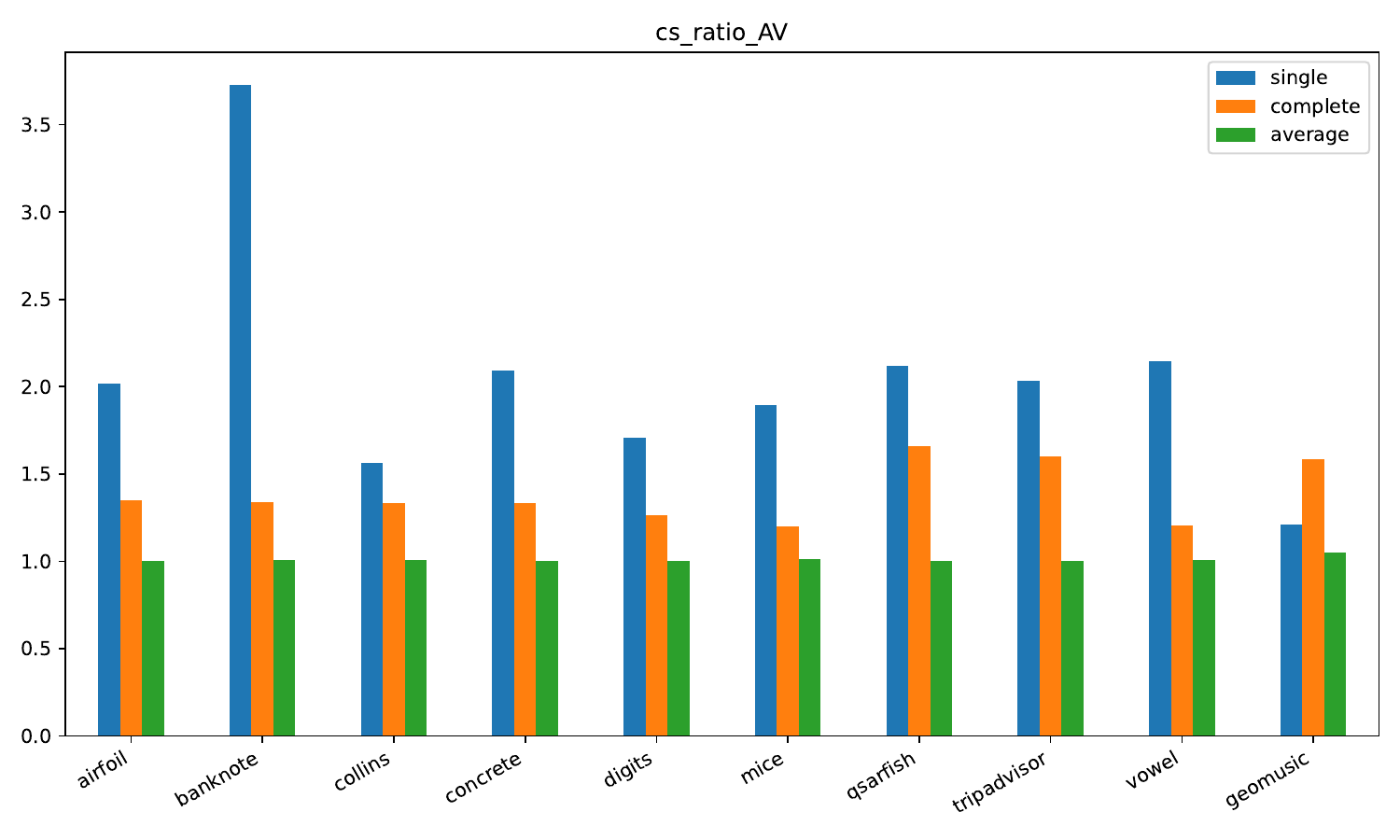}}
\caption{Results for the $\goodnessav$ \, for the different datasets and methods. For interpreting the bars, the lower  the better }
\label{fig:cs_ratio_AV}
\end{center}
\vskip -0.2in
\end{figure}

\begin{figure}[ht]
\vskip 0.2in
\begin{center}
\centerline{\includegraphics[width=0.9\columnwidth]{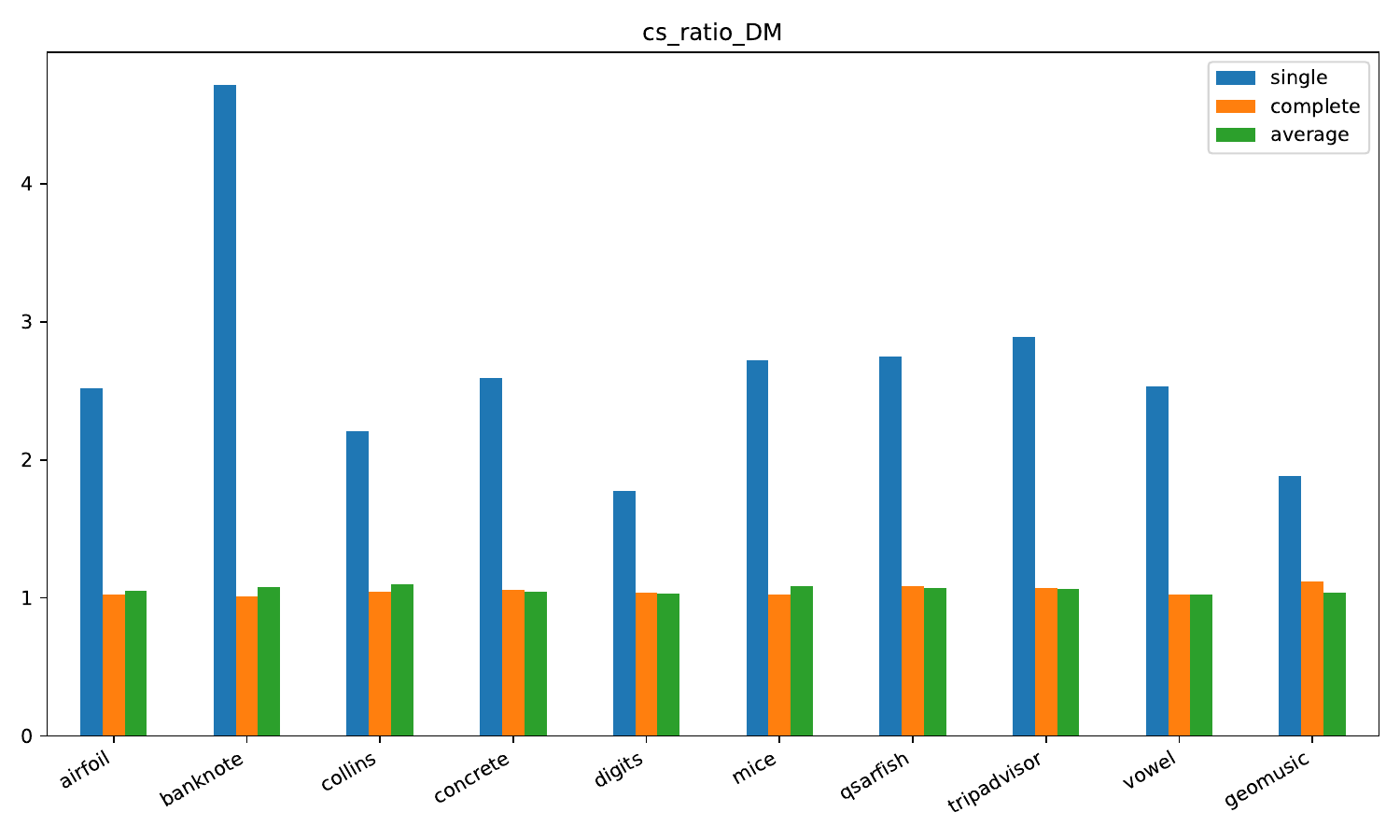}}
\caption{Results for the $\goodnessdm$ \, for the different datasets and methods. For interpreting the bars, the lower  the better }
\label{fig:cs_ratio_DM}
\end{center}
\vskip -0.2in
\end{figure}

Figures (\ref{fig:max_diam})-(\ref{fig:cs_ratio_DM}) show the results
obtained by \singlelink, \complink \, and \avglink, for all datasets and the different criteria considered in the paper. For a given dataset $D$, method $M$ and criterion $\alpha$, the height of the bar is given 
by  the average of $m_k$ 
 for every $k$  considered in our experiments, where $m_k$
 is the ratio between the value of criterion $\alpha$  achieved by method  $M$
 on dataset $D$ divided by the best value for
criterion $\alpha$,
 among  those achieved by \singlelink, \avglink\, and \complink  on dataset $D$. 

Regarding the cohesion criteria 
\complink\, presents the best results for $\maxdiamset$,  followed by \avglink.
For $\maxavgset$, again
\complink\, and \avglink\, are the best, with the latter having a slight advantage.

In terms of the separability criteria,  \avglink is much better than the other methods for $\costmin$, while for  
$\costavg$ there is a balance between \avglink\, and \singlelink.

For the criteria that combine cohesion and separability, \avglink is superior for $\goodnessav$, while there is a balance between 
\avglink\, and \complink\, for
$\goodnessdm$.

Table \ref{tab:l1} and \ref{tab:linfty} show the results
for the experiment described in Section \ref{sec:experiments}, when the
Euclidean distance is replaced with the $\ell_1$
 and $\ell_{\infty}$ norm, respectively.
The observations made in Section \ref{sec:experiments} also hold
when these metrics are used.

Finally, we note that the variance of the results for \avglink \, is small.
Indeed, an entry (average) close to 1 (e.g. 0.96) cannot have an underlying large variance 
because $1$ is the maximum possible value for an entry.
Since most entries for  \avglink \, are close to 1, 
one can conclude that the variance of its results is usually small.  In the
 supplemental material, we have   
.csv files with our full results.

\begin{table}[htpb]
\caption{Average ratio between the result of a method
and the best one for each criterion and each group of $k$. The best
results are bold-faced. Distances are computed using 
$\ell_{1}$ norm}
\label{tab:l1}
\begin{center}
\begin{tabular}{lccc||ccc||ccc}
       & \multicolumn{3}{c}{Smal} & \multicolumn{3}{c}{Medium} & \multicolumn{3}{c}{Large} \\ \hline
       & A      & C      & S      & A       & C       & S      & A       & C      & S      \\ \hline

$\costmin$    & \textbf{0,99} & 0,81          & 0,75 & \textbf{0,99} & 0,86          & 0,66       & \textbf{0,99} & 0,9           & 0,71          \\
$\costavg$    & \textbf{0,98} & 0,83          & 0,93 & \textbf{0,96} & 0,89          & \textbf{1} & 0,97          & 0,95          & \textbf{0,99} \\
$\maxdiamset$ & 0,86          & \textbf{0,99} & 0,72 & 0,85          & \textbf{1}    & 0,5        & 0,81          & \textbf{1}    & 0,41          \\
$\maxavgset$  & \textbf{0,94} & \textbf{0,94} & 0,88 & \textbf{0,99} & 0,9           & 0,73       & \textbf{0,99} & 0,83          & 0,7           \\
$\goodnessdm$ & \textbf{0,96} & 0,91          & 0,62 & 0,96          & \textbf{0,98} & 0,38       & 0,88          & \textbf{0,99} & 0,32          \\
$\goodnessav$ & \textbf{0,98} & 0,8           & 0,71 & \textbf{1}    & 0,79          & 0,51       & \textbf{1}    & 0,76          & 0,51         
\end{tabular}
\end{center}
\end{table}

\begin{table}[htpb]
\caption{Average ratio between the result of a method
and the best one for each criterion and each group of $k$. The best
results are bold-faced. Distances are computed using 
$\ell_{\infty}$ norm}
\label{tab:linfty}
\begin{center}
\begin{tabular}{lccc||ccc||ccc}
       & \multicolumn{3}{c}{Smal} & \multicolumn{3}{c}{Medium} & \multicolumn{3}{c}{Large} \\ \hline
       & A      & C      & S      & A       & C       & S      & A       & C      & S      \\ \hline
$\costmin$    & \textbf{0,99} & 0,82          & 0,77 & \textbf{0,98} & 0,91          & 0,7        & \textbf{0,99} & 0,94          & 0,75       \\
$\costavg$    & \textbf{0,97} & 0,82          & 0,95 & \textbf{0,97} & 0,92          & \textbf{1} & 0,98          & 0,96          & \textbf{1} \\
$\maxdiamset$ & 0,94          & \textbf{1}    & 0,9  & 0,87          & \textbf{1}    & 0,7        & 0,85          & \textbf{1}    & 0,56       \\
$\maxavgset$  & 0,94          & \textbf{0,96} & 0,91 & \textbf{0,94} & 0,88          & 0,79       & \textbf{0,95} & 0,85          & 0,81       \\
$\goodnessdm$ & \textbf{0,97} & 0,86          & 0,74 & 0,91          & \textbf{0,98} & 0,52       & 0,89          & \textbf{0,99} & 0,45       \\
$\goodnessav$ & \textbf{0,96} & 0,82          & 0,74 & \textbf{0,96} & 0,85          & 0,59       & \textbf{0,97} & 0,82          & 0,65 
\end{tabular}
\end{center}
\end{table}

\end{document}